%% file: main.tex
\documentclass[11pt]{article}
\usepackage{ifthen}
\usepackage{mdwlist}
\usepackage{amsmath,amssymb,amsfonts,amsthm}
\usepackage{bm}
\usepackage{hyperref}
\usepackage{subfigure}
\usepackage{caption}
\usepackage{graphicx}
\usepackage{xspace}
\usepackage{verbatim}
\usepackage{algorithm}
\usepackage{algpseudocode}
\usepackage[margin=1in]{geometry}
\usepackage{color}
\usepackage{thmtools}
\usepackage{thm-restate}
\usepackage{latexsym}
\usepackage{epsfig}
\usepackage{todonotes}
\usepackage{comment}
\input{locdef}

\input{glodef}

\usepackage{authblk}

\def\withcolors{0}
\def\withnotes{1}

\ifnum\withcolors=1
\newcommand{\new}[1]{{\color{red} {#1}}} 
\else
\newcommand{\new}[1]{{{#1}}}

\fi

\ifnum\withnotes=1

\else

\fi

\title{Context-Aware Local Differential Privacy}

\author[1]{Jayadev Acharya}
\author[2]{Keith Bonawitz}
\author[2]{Peter Kairouz}
\author[2]{Daniel Ramage}
\author[1]{Ziteng Sun\thanks{JA and ZS are supported by NSF-CCF-1846300 (CAREER), NSF-CRII-1657471 and a Google Faculty Research Award. Work partially done while ZS was visiting Google.}}
\affil[1]{ECE, Cornell University. \{acharya, zs335\}@cornell.edu}
\affil[2]{Google. \{bonawitz, kairouz, dramage\}@google.com }

\date{}
\begin{document}
\sloppy


\maketitle

\begin{abstract}
Local differential privacy (LDP) is a strong notion of privacy for individual users that often comes at the expense of a significant drop in utility. The classical definition of LDP assumes that all elements in the data domain are equally sensitive. However, in many applications, some  symbols are more sensitive than others. This work proposes a context-aware framework of local differential privacy that allows a privacy designer to incorporate the application's context into the privacy definition. For binary data domains, we provide a universally optimal privatization scheme and highlight its connections to Warner's randomized response (RR) and Mangat's improved response. Motivated by geolocation and web search applications, for $k$-ary data domains, we consider two special cases of context-aware LDP: block-structured LDP and high-low LDP. We study discrete distribution estimation and provide communication-efficient, sample-optimal schemes and information theoretic lower bounds for both models. We show that using contextual information can require fewer samples than classical LDP to achieve the same accuracy.

\end{abstract}
	
\input{intro}

\input{preliminaries}

\input{ldp_operational}

\input{binary}

\input{high_low}
\input{block}

\input{experiments}

\input{conclusion}

\newpage
\bibliographystyle{IEEEtran}
\bibliography{masterref}

\newpage
\appendix

\input{proof_properties}
\input{binary_proof}
\input{upper_proof}
\input{lower_bound}
\end{document}

%% file: locdef.tex
\newcommand{\partition}{\Pi}

\newcommand{\dtv}[2]{d_{TV}(#1, #2)}

\newcommand{\Epsilon}{E}

\newcommand{\X}{X}
\newcommand{\x}{x}
\newcommand{\Y}{Y}
\newcommand{\y}{y}

\newcommand{\Q}{Q}
\newcommand{\q}{q}

\newcommand{\ab}{k}

\newcommand{\ns}{n}
\newcommand{\dist}{\alpha}
\newcommand{\p}{p}

\newcommand{\eps}{\varepsilon}

%% file: glodef.tex
\usepackage{amsfonts,amsmath,amssymb,amsthm, pbox}

\usepackage{multirow}
\usepackage{pgfplots}
\usepgfplotslibrary{fillbetween}
\usetikzlibrary{patterns}
\usepackage{bigints}
\usepackage{color}

\usepackage[shortlabels]{enumitem}
\setitemize{noitemsep,topsep=0pt,parsep=0pt,partopsep=0pt}
\setenumerate{noitemsep,topsep=0pt,parsep=0pt,partopsep=0pt}

\makeatletter
\@ifundefined{theorem}{
  \theoremstyle{definition}
  \newtheorem{definition}{Definition}
  \theoremstyle{plain}
  \newtheorem{theorem}{Theorem}
  
  \newtheorem{proposition}{Proposition}
  \newtheorem{lemma}{Lemma}
  \newtheorem{claim}{Claim}

  \theoremstyle{remark}

}{}
\makeatother

\def \Paren#1{{\left({#1}\right)}}

\newcommand{\ignore}[1]{}

\newcommand{\EE}{\mathbb{E}}

\newcommand{\RR}{\mathbb{R}}

\newcommand{\expectation}[1]{\EE\left[#1\right]}
\newcommand{\variance}[1]{Var\left(#1\right)}

\def \cA     {{\cal A}}

\def \cP     {{\cal P}}
\def \cQ     {{\cal Q}}
\def \cR     {{\cal R}}

\def \cW     {{\cal W}}
\def \cX     {{\cal X}}
\def \cY     {{\cal Y}}
\def \cZ     {{\cal Z}}

\def \upto  {{,}\ldots{,}}

\def \ceil#1{{\lceil{#1}\rceil}}

\def \Paren#1{{\left({#1}\right)}}

\newcommand{\probof}[1]{\Pr\Paren{#1}}

\def\ignore#1{}

\newcommand{\bi}{\begin{itemize}}
\newcommand{\ei}{\end{itemize}}

\def\orpro{\mathop{\mathchoice
   {\vee\kern-.49em\raise.7ex\hbox{$\cdot$}\kern.4em}
   {\vee\kern-.45em\raise.63ex\hbox{$\cdot$}\kern.2em}
   {\vee\kern-.4em\raise.3ex\hbox{$\cdot$}\kern.1em}
   {\vee\kern-.35em\raise2.2ex\hbox{$\cdot$}\kern.1em}}\limits}

\def\andpro{\mathop{\mathchoice
 {\wedge\kern-.46em\lower.69ex\hbox{$\cdot$}\kern.3em}
 {\wedge\kern-.46em\lower.58ex\hbox{$\cdot$}\kern.25em}
 {\wedge\kern-.38em\lower.5ex\hbox{$\cdot$}\kern.1em}
 {\wedge\kern-.3em\lower.5ex\hbox{$\cdot$}\kern.1em}}\limits}

\def\simge{\mathrel{%
   \rlap{\raise 0.511ex \hbox{$>$}}{\lower 0.511ex \hbox{$\sim$}}}}

\def\simle{\mathrel{
   \rlap{\raise 0.511ex \hbox{$<$}}{\lower 0.511ex \hbox{$\sim$}}}}



%% file: intro.tex
\section{Introduction}
Differential privacy (DP) is a rigorous notion of privacy that enforces a worst-case bound on the privacy loss due to release of query results ~\cite{Dwork06}. Its local version, local differential privacy (LDP) (Definition~\ref{def:local_dp}), provides context-free privacy guarantees even in the absence of a trusted data collector \cite{warner1965randomized, BeimelNO08, KasiviswanathanLNRS11}. Under LDP, all pairs of elements in a data domain are assumed to be equally sensitive, leading to harsh privacy-utility trade-offs in many learning applications. In many real-life settings, however, some elements are more sensitive than others. For example, in URL applications, users may want to differentiate sensitive URLs from non-sensitive ones, and in geo-location applications, users may want to hide their precise location within a city, but not the city itself. 

This work introduces a unifying context-aware notion of LDP where different pairs of domain elements can have ``arbitrary'' sensitivity levels. For binary domains, we provide a universal optimality result and highlight interesting connections to Warner's response and Mangat's improved response. For $k$-ary domains, we look at two canonical examples of context-aware LDP: block-structured LDP and high-low LDP. For block-structured LDP, the domain is partitioned into $m$ blocks and the goal is to hide the identity of elements within the same block but not the block identity. This is motivated by geo-location applications where users can be in different cities, and it is not essential to hide the city of a person but rather where exactly within that city a person is (i.e., which bars, restaurants, or businesses they visit). In other words, in this context, users would like to hide which element in a given block of the domain their data belongs to -- and not necessarily which block their data is in, which may be known from side information or application context. For high-low LDP, we assume there is a set of sensitive elements and we only require the information about possessing sensitive elements to be protected. This can be applied in web browsing services, where there are a large number of URLs and not all of them contain sensitive information.
\subsection{Related work}

Recent works consider relaxations of DP~\cite{chatzikokolakis2013broadening, doudalis2017one, borgs2018revealing, asi2019element, feyisetan2019leveraging, dong2019gaussian} and LDP~\cite{chatzikokolakis2017efficient, Alvim2018InvitedPL,murakami2018restricted, kawamoto2019local, geumlek2019profile} that incorporate the context and structure of the underlying data domain. Indeed, \cite{doudalis2017one, murakami2018restricted} investigate settings where a small fraction of the domain elements are secrets and require that the adversary's knowledge about whether or not a user holds a secret cannot increase much upon the release of data.~\cite{chatzikokolakis2013broadening, borgs2018revealing, chatzikokolakis2017efficient, Alvim2018InvitedPL, xiang2019linear} model the data domain as a metric space and scale the privacy parameter between pairs of elements by their distance.~\cite{ScheinWSZW19} considers categorical high dimensional count models and define privacy measures that help privatize counts that are close in $\ell_1$ distance. 

To investigate the utility gains under this model of privacy, we consider the canonical task of distribution estimation:
 Estimate an unknown discrete distribution given (privatized) samples from it. The trade-off between utility and privacy in distribution estimation has received recent attention, and optimal rates have been established~\cite{duchi2013local, DiakonikolasHS15, kairouz2016discrete, WangHWNXYLQ16, YeB17, AcharyaSZ19}. These recent works show that the sample complexity for $k$-ary distribution estimation up to accuracy $\alpha$ in total variation distance increases from $\Theta \Paren{\frac{k}{\alpha^2}}$ (without privacy constraints) to $\Theta\Paren{\frac{k^2}{\alpha^2 \eps^2}}$ (for $\eps=O(1)$) when we impose $\eps$-LDP constraints. For both the block-structured and high-low models of LDP, we will characterize the precise sample complexity by providing privatization and estimation schemes that are both computation- and communication-efficient. 

Another line of work considers a slightly different problem of heavy hitter estimation where there is no distributional assumption under the data samples that the users have and the main focus is on reducing the computational complexity and communication requirements~\cite{bassily2015local, BassilyNST17, hsu2012distributed, wang2017locally, BunNS18, zhu2019federated, AcharyaS19}.

\subsection{Our contributions}
Motivated by the limitations of LDP, and practical applications where not all domain elements are equally sensitive, we propose a general notion of \emph{context-aware} local differential privacy (see Definition~\ref{def:general_local_dp}) and prove that it satisfies several important properties, such as preservation under post-processing, adaptive composition, and robustness to auxiliary information. 

When the underlying data domain is binary, we provide a complete characterization of a universally optimal scheme, which interpolates between Warner's randomized response and Mangat's improved response, given in Theorem~\ref{thm:optimal_binary}.

We then consider general data domains and investigate two practically relevant models of context-aware LDP. The first is the \emph{block-structured} model and is motivated by applications in geo-location and census data collection. Under this privacy model, we assume that the underlying data domain is divided into a number of blocks, and the elements within a block are sensitive. For example, when the underlying domain is a set of geographical locations, each block can correspond to the locations within a city. In this case, we would like to privatize the precise value within a particular block, but the privacy of the block ID is not of great concern (Definition~\ref{def:bsldp}). In Theorem~\ref{thm:bs_complexity} we characterize the sample complexity of estimating $k$-ary discrete distributions under this privacy model. We also propose a privatization scheme based on the recently proposed Hadamard Response (HR), which is both computation- and communication-efficient. We then prove the optimality of these bounds by proving a matching information-theoretic lower bound. This is achieved by casting the problem as an information constrained setting and invoking the lower bounds from~\cite{pmlr-v99-acharya19a, AcharyaCFT19}.  \textcolor{black}{We show that when all the blocks have roughly the same number of symbols, the sample complexity can be reduced by a factor of the number of blocks compared to that under the classic LDP.} See Theorem \ref{thm:bs_complexity} for a formal statement of this result.

The second model we consider is the \emph{high-low} model, where there are a few domain elements that are sensitive while the others are not. A form of this high-low notion of privacy was proposed in~\cite{murakami2018restricted}, which also considered distribution estimation and proposed an algorithm based on RAPPOR~\cite{erlingsson2014rappor}. We propose a new privatization scheme based on HR that has the same sample complexity as~\cite{murakami2018restricted} with a much smaller communication budget. We also prove a matching information theoretic lower bound showing the optimality of these schemes~\cite{murakami2018restricted}. \textcolor{black}{In this case, the sample complexity only depends quadratically in the number of sensitive elements and linearly in the domain size compared to the quadratic dependence on the domain size under the classic LDP.} See Theorem \ref{thm:hl_complexity} for a formal statement of this result.

As a consequence of these results, we observe that the sample complexity of distribution estimation can be significantly less than that in the classical LDP. Thus contextual privacy should be viewed as one possible method to improve the trade-off between the utility and privacy in LDP when different domain elements have different levels of privacy.

To validate our worst-case (minimax) analyses, we conduct experiments on synthetic data sampled from uniform, geometric, and power law distributions in addition to experiments on over 3.6 million check-ins to 43,750 locations from the Gowalla dataset. Our experiments confirm that context-aware LDP achieves a far better accuracy under the same number of samples.

\subsection{Organization}
In Section~\ref{sec:preliminaries} we define LDP and the problem of distribution estimation. In Section~\ref{sec:context} we define context-aware LDP and provide an operational definition along with some of its specializations. In Section~\ref{sec:binary} we provide the optimal privatization scheme for binary domains. In Section~\ref{sec:hilo} and Section~\ref{sec:bsldp} we derive the optimal sample complexity in the high-low model, and block-structured model respectively. Experimental results are presented in Section~\ref{sec:experiments}. We conclude with a few interesting extensions in Section \ref{sec:conc}.

%% file: preliminaries.tex
\vspace{-6pt}
\section{Preliminaries}
\label{sec:preliminaries}
Let $\cX$ be the $k$-ary underlying data domain, wlog let $\cX=[\ab]:=\{1,\ldots,\ab\}$. 
There are $n$ users, and user $i$ has a (potentially) sensitive data point $X_i\in\cX$. 
A privatization scheme is a mechanism to add noise to mask the true $X_i$'s, and can be represented by a conditional distribution
$\Q\footnote{\new{In the remaining parts of this paper, we use terms privatization scheme, conditional distribution, and channel to refer to this $Q$ matrix interchangeably.}}:\cX\to\cY$, where $\cY$ denotes the output domain of the privatization scheme. For $y\in\cY$, and $x\in\cX$, $\Q\left(\y\mid\x\right)$ is the probability that the privatization scheme outputs $y$, upon observing $x$. If we let $X$ and $Y$ denote the input and output of a privatization scheme, then $\Q\left(\y\mid\x\right) = \probof{\Y =\y\mid\X = \x}$. For any set $S \subset \cY$, we have $Q(S \mid x) := \probof{Y \in S \mid X = x}$. 

\begin{definition}[Local Differential Privacy.~\cite{warner1965randomized, KasiviswanathanLNRS11}]
	\label{def:local_dp} Let $\eps>0$. A conditional distribution $\Q$
	is  $\eps$-locally differentially private ($\eps$-LDP) if for all $\x$,
	$\x^\prime\in\cX$ and all $S \subset \cY$, 
	\begin{equation}
	\Q\left(S\mid\x\right) \leq  e^{\varepsilon} \Q\left(S\mid\x^\prime\right).
	\label{eqn:ldp}
	\end{equation}
	Let $\cQ_{\eps}$ be the set of all  $\eps$-LDP conditional distributions with input $[\ab]$.
\end{definition}

\noindent \textbf{Distribution estimation.} Let
\[
\Delta_\ab := \{(\p_1,\ldots,\p_k):  \forall  x \in [k], \p_x \ge 0;  \sum \p_x=1 \}
\] be all discrete distributions over $[\ab]$. We assume that the user's data $X^\ns:=X_1, \ldots, X_\ns$ are independent draws from an unknown $\p \in \Delta_\ab$. User $i$ passes $X_i$ through the privatization channel $Q$ and sends the output $Y_i$ to a central server. Upon observing the messages $Y^\ns:=Y_1,\ldots, Y_\ns$, the server then outputs $\hat{\p}$ as an estimate of $\p$. Our goal is to select $Q$ from a set of \emph{allowed channels} $\cQ$ and to design an estimation scheme $\hat{\p}:\cY^\ns\to\Delta_k$ that achieves the following min-max risk
\begin{equation} \label{eqn:risk_def}
	r(k,n,d,\cQ) = \min_{Q \in \cQ} \left\{ \min_{\hat{\p}} \max_{\p \in \triangle_k} \expectation{d\Paren{\p, \hat{\p}}} \right\},
\end{equation}
where $d\Paren{\cdot,\cdot}$ is a measure of distance between distributions. In this paper we consider the total variation distance, $\dtv{p}{q} := \frac12\sum_{i = 1}^k |\p_i - \q_i|$. For a parameter $\alpha>0$, the sample complexity of distribution estimation to accuracy $\alpha$ is the smallest number of users for the min-max risk to be smaller than $\alpha$,
\[
	n(k,\alpha,\cQ) := \arg \min_n \{ r(k,n,d_{\rm TV},\cQ)< \alpha \}.
\]
When $\cQ$ is $\cQ_{\eps}$, the channels satisfying $\eps$- LDP constraints with input domain $[k]$, it is now well established that for $\eps=O(1)$~\cite{duchi2013local, kairouz2016discrete, YeB17, WangHWNXYLQ16, pmlr-v99-acharya19a, AcharyaSZ19},
\begin{align}
	n(k,\alpha,\cQ_{\eps}) = \Theta \Paren{\frac{k^2}{\alpha^2\eps^2}}.\label{eqn:optimal-ldp}
\end{align}
\textcolor{black}{
\noindent \textbf{Hadamard matrix (Sylvester's construction).} Let $H_1 = [1]$, then Sylvester's construction of Hadamard matrices is a sequence of square matrices of size $2^i \times 2^i$ recursively defined as
\[
    H_{2m} = 
    \begin{bmatrix}
    H_m & H_m \\
    H_m & -H_m
    \end{bmatrix}
\]
Letting $S_i = \{y \mid y \in [m], H(i+1,y) = +1\}$ be the column indices of $+1$'s in the $(i+1)$th row of $H_m$, we have
\begin{itemize}
    \item $\forall i \in [m-1], |S_i| = \frac{m}{2}$.
    \item $\forall i \neq j \in [m-1], |S_i \cap S_j| = \frac{m}{4}$
\end{itemize}
}

%% file: ldp_operational.tex
\vspace{-6pt}
\section{Context-Aware LDP}\label{sec:context}
In local differential privacy, all elements in the data domain are assumed to be equally sensitive, and the same privacy constraint is enforced on all pairs of them (see~\eqref{eqn:ldp}).
However, in many settings some domain elements might be more sensitive than others.
To capture this, we present a context-aware notion of privacy. Let $\Epsilon\in \RR_{\ge0}^{\ab\times\ab}$ be a matrix of non-negative entries, where for $x, x^\prime\in[k]$, $\eps_{\x,\x^\prime}$ is the $(x,x^\prime)$th entry of $E$. 
\begin{definition}[Context-Aware LDP]
	\label{def:general_local_dp} A conditional distribution $\Q$
	is $\Epsilon$-LDP if for all $\x$,
	$\x^\prime\in\cX$ and $S \subseteq \cY$,
	\begin{equation}
	\label{eq:general_local_dp}
	    \Q\left(S\mid\x\right) \leq  e^{\varepsilon_{\x,\x^\prime}} \Q\left(S\mid\x^\prime\right).
	\end{equation}
	\vspace{-15pt}
\end{definition}
\textcolor{black}{This definition allows us to have a different sensitivity level between each pair of elements to incorporate context information.} For a \emph{privacy matrix} $E$, the set of all $\Epsilon$-LDP mechanisms is denoted by ${\cQ}_{\Epsilon}$. When all the entries of $\Epsilon$ are $\eps$, we obtain the classical LDP. 

\medskip
\new{
\noindent\textbf{Symmetric $E$ matrices.}
When the $E$ matrix is symmetric, context-aware LDP introduces a similar structure as Metric-based LDP~\cite{Alvim2018InvitedPL}, which considers $\cX$ to be a metric space endowed with a distance $d_\cX$. In this case, consider
\[
	\eps_{\x, \x^\prime} = \eps_{\x^\prime, \x} = d_{\cX} (\x, \x^\prime).
\]
It is required that it is harder to distinguish close-by symbols compared to symbols that are relatively far from each other. As a special case, we introduce (later in this section) the notion of Block-Structured LDP as an example which only requires the elements to be indistinguishable if they are in the same block.
}

\medskip
\new{
\noindent\textbf{Asymmetric $E$ matrices.}
An advantage of our framework is that it allows the matrix $E$ to be asymmetric.  Consider a binary setting where we ask a user whether or not they have used drugs before. Here, people whose answer is \textbf{no} can occasionally say \textbf{yes} to protect those whose answer is \textbf{yes}. Thus, when the data collector sees a \textbf{yes}, they are lost as to whether it came from a true \textbf{yes} answer or a \textbf{no} that was flipped to a \textbf{yes}. Further, there is no need to randomize \textbf{yes} answers because it is okay for the data collector to know who did not do drugs. This is captured in our framework by allowing $Q(\text{\textbf{no}} \mid \text{\textbf{no}})/Q(\text{\textbf{no}} \mid \text{\textbf{yes}})$ to be arbitrarily large and placing an upper bound on $Q(\text{\textbf{yes}} \mid \text{\textbf{yes}})/Q(\text{\textbf{yes}} \mid \text{\textbf{no}})$. A well-known scheme that satisfies this requirement is Mangat’s improved randomized response~\cite{mangat1994improved}. See Section 4 for more details. This intuition is generalized to $k$-ary alphabet in the High-Low LDP model introduced later in this section.
}

We provide an operational (hypothesis testing) interpretation for context-aware LDP in Section~\ref{sec:operational}. We show that context-aware LDP is robust to post-processing, robust to auxiliary information, and is adaptively composable in Section~\ref{sec:properties}.

For binary domains, namely when $k=2$, we give a single optimal privatization scheme for all $\Epsilon$ matrices in Section~\ref{sec:binary}. For general $k$, it is unclear whether there is a simple characterization of the optimal schemes for all $\Epsilon$ matrices. First note that if $\eps^* = \min_{i,j} \eps_{i,j}$ is the smallest entry of $\Epsilon$, then any $\eps^*$-LDP algorithm is also $\Epsilon$-LDP. However, this is not helpful since it does not help us capture the context of the application and consequently get rid of the stringent requirements of standard LDP. Motivated by applications in geo-location and web search, we consider structured $E$ matrices that are both practically motivating, and are paramterized by a few parameters (or a single parameter) so that they are easier to analyze  and implement.

\medskip
\noindent\textbf{High-Low LDP (HLLDP).}
The high-low model captures applications where there are certain domain elements that are private, and the remaining elements are non-private. We only want to protect the privacy of private elements. This is formalized below.
\begin{definition} \label{def:high-low}
Let $A = \{\x_1, \cdots, \x_s\}\subset\cX$ denote the set of sensitive domain elements, and all symbols in $B := \cX\setminus A$ are non-sensitive. A privatization scheme is said to be $(A, \eps)$-HLLDP if $\forall S \subset \cY$, and $x\in A, x' \in \cX$,
\begin{align}
	Q(S\mid x)\le e^{\eps}Q(S\mid x'), \label{eqn:hilo}
\end{align}
\new{which corresponds to the following $E$ matrix:
\begin{equation*}
    \eps_{x, x'} = 
    \begin{cases}
        \eps, & x \in A, \\
        \infty, & x \in B.
    \end{cases}
\end{equation*}}
\vspace{-15pt}
\end{definition}
This implies that when the input symbol is in $A$, the output distribution cannot be multiplicatively larger than the output distribution for any other symbol, but there is no such restriction for symbols in $B$.
HLLDP was also defined in~\cite{murakami2018restricted}. We solve the problem of minimax distribution estimation under this privacy model in Section~\ref{sec:hilo}.

\noindent\textbf{Block-Structured LDP (BSLDP).}
In applications such as geo-location, it is important to preserve the privacy of symbols that are close to each other. We consider this model where the data domain is divided into various blocks (e.g., the cities), and we would like the symbols within a block (e.g., the various locations within a city) to be hard to distinguish. This is formalized below.

\begin{definition}
	Suppose there is a partition of $\cX$, which is $\partition = \{ \cX_1, \cX_2, ..., \cX_m \}$. \new{With a slight abuse of notation, we define $\forall x \in \cX_i, \partition(x) := i$.} Then $\forall \eps > 0$, a privatization scheme is said to be $(\partition, \eps)$ - BSLDP if it satisfies $\forall x, x' \in \cX_i$ such that $\partition(x) = \partition(x')$, and any $S \subset \cY$, we have
	\begin{equation}
		Q(S\mid x)\le e^{\eps}Q(S\mid x'),
	\end{equation}
	\new{which corresponds to the following $E$ matrix:
\begin{equation*}
    \eps_{x, x'} = 
    \begin{cases}
        \eps, & \partition(x) = \partition(x'), \\
        \infty, & \partition(x) \neq \partition(x').
    \end{cases}
\end{equation*}}
\vspace{-15pt}
	\label{def:bsldp}
\end{definition}
This definition relaxes the local version of differential privacy in the following way. Given a partition of the input set $\partition = \{ \cX_1, \cX_2, ..., \cX_m \}$, it requires different levels of indistinguishablitity for element pairs in the same block and those in different blocks. 
We solve the problem of minimax distribution estimation under this model in Section~\ref{sec:bsldp}.

\subsection{Operational interpretation of context-aware LDP}
\label{sec:operational}
Recall that the entries of $\Epsilon$ denote the amount of privacy across the corresponding row and column symbols. We provide an operational interpretation of $E$-LDP by considering a natural hypothesis testing problem. Suppose we are promised that the input is in $\{x, x^\prime\}$ for some symbols $x,x^\prime\in[k]$, and an $E$-LDP scheme outputs a symbol $\Y\in\cY$. Given $Y$, the goal is to test whether the input is $x$ or $x^\prime$. 
\begin{theorem}[Operational Interpretation of Context-Aware LDP]
	\label{thm:operational_def_gdp} A conditional distribution $\Q$
	is $\Epsilon$-locally differentially private if and only if for all $x, x^\prime \in \cX$ and all  decision rules $\hat{\X}: \cY \rightarrow \{x,x^\prime\}$,
	\begin{align}
	P_{\rm FA}(x,x^\prime)  + e^{\eps_{x^\prime,x}}P_{\rm MD}(x,x^\prime)  \geq  1, \label{eqn:famd} \\
	e^{\eps_{x,x^\prime}}P_{\rm FA}(x,x^\prime)  + P_{\rm MD}(x,x^\prime)  \geq  1, \label{eqn:mdfa} 
	\end{align}
	where $P_{\rm FA}(x, x^\prime) = \Pr({\hat{\X} = x\mid \X = x^\prime})$ \new{is the false alarm rate} and $ P_{\rm MD}(x, x^\prime) = \Pr({\hat{\X} = x^\prime\mid \X = x})$ \new{is the miss detection rate}. 
\end{theorem}

The proof is provided in Section~\ref{sec:operational_def_gdp}. Consider a test for distinguishing $x$ and $x^\prime$ given above in Theorem~\ref{thm:operational_def_gdp}. Figure~\ref{fig:err_region} shows the effective error regions for any estimator $\hat{\X}$ under the privacy constraints $\eps_{x,x^\prime}$ and $\eps_{x^\prime,x}$. We can see that unlike the symmetric region under LDP, we are pushing the miss detection rate to be higher when $\eps_{x, x^\prime} < \eps_{x^\prime,x}$. This shows that symbol $x$ is more private than symbol $x^\prime$, namely we want to protect the identity of symbol being $x$ more than we want to protect the identity being $x^\prime$.

\input{fig_err_region}

\subsection{Properties of context-aware LDP} 
\label{sec:properties}

Context-aware LDP also satisfies several important properties held by classical LDP, including post-processing, adaptive composition and robustness to auxiliary information. We provide their proofs in Section~\ref{sec:proof_properties}.

\begin{proposition}[Preservation under post-processing] \label{prop:post-processing}
    Let $\cA_1: \cX \rightarrow \cY_1$ be an $E$-LDP scheme and $\cA_2: \cY_1 \rightarrow \cY_2$ is any algorithm that processes the output of $\cA_1$, then the scheme $\cA = \cA_2 \circ \cA_1$ is also $E$-LDP.
\end{proposition}

\begin{proposition}[Adaptive composition] \label{prop:adapt_compose}
    Let $\cA_1: \cX \rightarrow \cY_1$ be an $E_1$-LDP scheme and $\cA_2: \cX \times \cY_1 \rightarrow \cY_2$ be an $E_2$-LDP scheme, then the scheme $\cA$ defined as $(\cA_1, \cA_2)$ is $(E_1 + E_2)$-LDP.
\end{proposition}

\begin{proposition}[Robustness to auxiliary information] \label{prop:side_information}
    Let $\p^*$ be a prior we have over $\cX$ and $\cA: \cX \rightarrow \cY$ be an $E$-LDP scheme. Then $\forall x_1, x_2 \in \cX$ and $y \in \cY$,
    \[
        \frac{\probof{X = x_1\mid Y = y}}{\probof{X = x_2\mid Y = y}} \le e^{\eps_{x_1,x_2}} \frac{\p^*(x_1)}{\p^*(x_2)}.
    \]
\end{proposition}

%% file: fig_err_region.tex
\begin{centering}

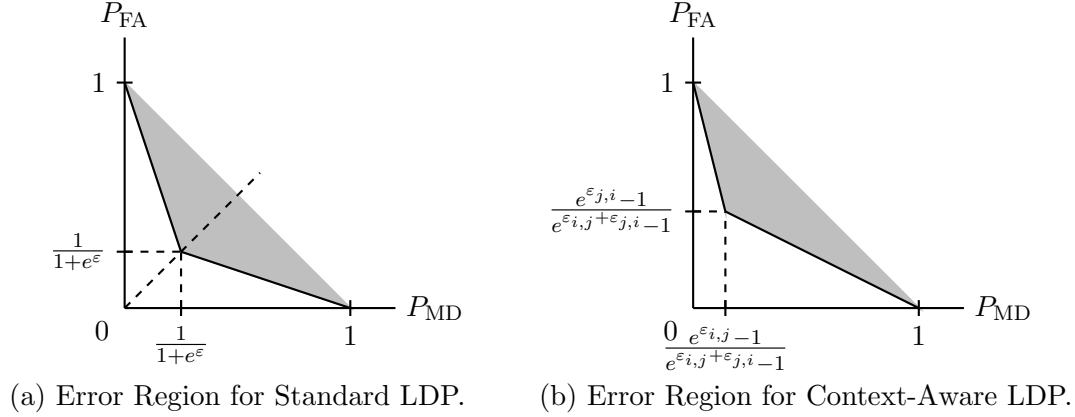
\begin{figure*}
\begin{minipage}[t]{.45\textwidth}
\centering
\begin{tikzpicture}[thick, scale = 3.0]
\draw[draw=gray!50!white,fill=gray!50!white] 
    plot[smooth,samples=100,domain=0:0.25] (\x,{1 - \x}) -- 
    plot[smooth,samples=100,domain=0.25:0] (\x,{1-3*\x});

\draw[draw=gray!50!white,fill=gray!50!white] 
    plot[smooth,samples=100,domain=0.25:1] (\x,{1 - \x}) -- 
    plot[smooth,samples=100,domain=1:0.25] (\x,{(1-\x)/3});

\draw[domain=0:0.25] plot (\x,{1 - 3*\x});

\draw[domain=0.25:1] plot (\x,{(1 - \x)/3});

\draw (0,0)--(1.2,0) node[right]{$P_{\rm MD}$};
\draw (0,0)--(0,1.2) node[above]{$P_{\rm FA}$};
\foreach \x in {1}
    \draw (\x,1pt)--(\x,-1pt) node[below] {$\x$};
\foreach \y/\ytext in {1}
    \draw (1pt,\y)--(-1pt,\y) node[left] {$\y$};  
    
\draw [dashed] (0,0)--(0.6,0.6);
\draw (1pt,0.25)--(-1pt,0.25) node[left] {$\frac{1}{1+e^\eps}$}; 
\draw (0.25,1pt)--(0.25,-1pt) node[below] {$\frac{1}{1+e^\eps}$};
\draw [dashed] (0.25,0) -- (0.25, 0.25);
\draw [dashed] (0,0.25) -- (0.25, 0.25);
\node at (-0.1, -0.1) {$0$};
\node at (0.5, -0.4) {(a) Error Region for Standard LDP.};
\end{tikzpicture}
\end{minipage}
\begin{minipage}[t]{.45\textwidth}
\centering
\begin{tikzpicture}[thick, scale = 3.0]
\draw[draw=gray!50!white,fill=gray!50!white] 
    plot[smooth,samples=100,domain=0:1/7] (\x,{1 -\x}) -- 
    plot[smooth,samples=100,domain=1/7:0] (\x,{1-4*\x});

\draw[draw=gray!50!white,fill=gray!50!white] 
    plot[smooth,samples=100,domain=1/7:1] (\x,{1 - \x}) -- 
    plot[smooth,samples=100,domain=1:1/7] (\x,{(1-\x)/2});

\draw[domain=0:1/7] plot (\x,{1 - 4*\x});

\draw[domain=1/7:1] plot (\x,{(1 - \x)/2});

\draw (0,0)--(1.2,0) node[right]{$P_{\rm MD}$};
\draw (0,0)--(0,1.2) node[above]{$P_{\rm FA}$};
\foreach \x in {1}
    \draw (\x,1pt)--(\x,-1pt) node[below] {$\x$};
\foreach \y/\ytext in {1}
    \draw (1pt,\y)--(-1pt,\y) node[left] {$\y$};

\draw (1pt,3/7)--(-1pt,3/7) node[left] {$\frac{e^{\eps_{j,i}} -1}{e^{\eps_{i,j}+\eps_{j,i}}-1}$}; 
\draw (1/7,1pt)--(1/7,-1pt) node[below] {$\frac{e^{\eps_{i,j}} -1}{e^{\eps_{i,j}+\eps_{j,i}}-1}$};
\draw [dashed] (1/7,0) -- (1/7, 3/7);
\draw [dashed] (0,3/7) -- (1/7, 3/7);
\node at (-0.1, -0.1) {$0$};
\node at (0.5, -0.4) {(b) Error Region for Context-Aware LDP.};
\end{tikzpicture}
\end{minipage}
\caption{Error region for hypothesis testing between $i$ and $j$ under DP constraints}
\label{fig:err_region}
\end{figure*}

\end{centering}

%% file: binary.tex
\section{Binary Domains}
\label{sec:binary}
Consider a binary domain, namely $k=2$ and domain elements $\{1,2\}$. In this case, we have
	\[
	    E = \begin{bmatrix}
	    0 & \eps_{1,2} \\
	    \eps_{2,1} & 0
	    \end{bmatrix}.
	\]
Perhaps the oldest and simplest privatization mechanism is Warner's randomized response for confidential survey interviews~\cite{warner1965randomized}. In this section, we give the optimal scheme for all utility functions that obey the \textit{data processing inequality} under all possible binary constraints. We prove that when $\varepsilon_{1,2} = \varepsilon_{2,1}$, the optimal scheme is Warner's randomized response.

Consider the composition of two privatization mechanisms $\Q \circ W$ where
the output of the first mechanism $\Q$ is applied to another mechanism $W$.
We say that a utility function $U(\cdot)$ obeys the data processing inequality if for all $\Q$ and $W$
\vspace{-5pt}
\begin{eqnarray*}
	U(\Q \circ W)& \leq& U(\Q) \;.
\end{eqnarray*}

In other words, further processing of the data can only reduce the utility. Such utility functions are ubiquitous. \textcolor{black}{For example, in the minimax distribution learning context of this paper, $U(Q)$ may be chosen as $ - \min_{\hat{\p}} \max_{\p} \expectation{d\Paren{\p, \hat{\p}}}$ (i.e., the negative of the minimax risk under a fixed mechanism $Q$) with $d$ being any $\ell_p$ distance or $f$-divergence.}

\begin{theorem}[Optimal Mechanism for Binary Domains]
	\label{thm:optimal_binary}
	Let $\mathcal{X}$ be a binary input alphabet and $U(\Q)$ be any utility function that obeys the data processing inequality. Then for any
	$\varepsilon_{1,2}, \varepsilon_{2,1} \geq 0$, the following privatization
	mechanism
	\begin{equation} \label{eqn:binary_optimal}
	\Q^{*}=\frac{\left[
	\begin{array}{cc}
	 e^{\varepsilon_{2,1}}-1  & 1-e^{-\varepsilon_{1,2}}\\
	e^{-\varepsilon_{1,2}}\left(e^{\varepsilon_{2,1}}-1\right) & e^{\varepsilon_{2,1}}\left(1-e^{-\varepsilon_{1,2}}\right) \\
	\end{array}
	\right]}{e^{\varepsilon_{2,1}}-e^{-\varepsilon_{1,2}}}
	\end{equation}
	solves
	\begin{equation*}
	\begin{aligned}
	& \underset{\Q}{\rm{maximize} } &  & U(\Q) 
	& \text{\rm{subject to}} & & \Q \in \cQ_{E}.
	\end{aligned}
	\end{equation*}
	Here $\forall x,y \in \{1, 2\}, Q^*(x,y) := Q^*(y \mid x)$.
\end{theorem}
As a special case of the above theorem, if we consider the original local differential privacy setup where
$\varepsilon_{1,2}=\varepsilon_{2,1}=\varepsilon$, then the optimal mechanism
for binary alphabets is
\begin{equation}
\label{eq:optimal_binary}
\Q^{*}=\frac{1}{e^{\varepsilon}+1}\left[
\begin{array}{cc}
e^{\varepsilon} & 1 \\
1 & e^{\varepsilon} \\
\end{array}
\right].
\end{equation}
This is Warner's randomized response model in confidential structured survey interview with $p = e^\varepsilon/(e^\varepsilon+1)$ \cite{warner1965randomized}. Warner's randomized response was shown to be optimal for binary alphabets in \cite{kairouz2014extremal}. Another interesting special case of the above theorem is Mangat's improved randomized
response strategy~\cite{mangat1994improved}. To see this, let
$\varepsilon_{1,2}=\infty$ and $p=e^{-\varepsilon_{2,1}}$. Then
\begin{equation}
Q^{*}=\left[
\begin{array}{cc}
1-p & p \\
0 & 1 \\
\end{array}
\right].
\end{equation}
This is exactly Mangat's improved randomized response strategy. Thus Mangat's randomized response with 
$p=e^{-\varepsilon_{2,1}}$ is optimal for all utility
functions obeying the data processing inequality under this generalized differential privacy framework with
$\varepsilon_{1,2}=\infty$, \new{which corresponds to the case where element 1 is not sensitive}.

%% file: high_low.tex
\section{Distribution Estimation under HLLDP}
\label{sec:hilo}
We characterize the optimal sample complexity of distribution estimation under the high-low model (see Definition~\ref{def:high-low}).
\begin{theorem} \label{thm:hl_complexity}
Let $A\subset\cX$, with $|A|=s<k/2$, and $\eps=O(1)$.	Let $\cQ_{A, \eps}$ be the set of all possible channels satisfying $(A,\eps)$-HLLDP, then:
\begin{equation}
	n(k,\alpha,\cQ_{A,\eps}) = \Theta \Paren{\frac{s^2}{\alpha^2 \eps^2} + \frac{k}{\alpha^2 \eps}}. \label{eqn:hl_sample}
\end{equation}
\end{theorem}

When the size of the sensitive set is relatively large, e.g. $s > \sqrt{k}$, the sample complexity is $\Theta(s^2/(\alpha^2 \eps^2))$, which corresponds to classic LDP with alphabet size $s$. This question was considered in~\cite{murakami2018restricted}, which gave an algorithm based on RAPPOR~\cite{erlingsson2014rappor} that has the optimal sample complexity, but requires $\Omega(k)$ bits of communication from each user, which is prohibitive in settings where the uplink capacity is limited for users. We design a scheme based on Hadamard Response (HR), which is also sample-optimal but requires only $\log k$ bits of communication from each user. 

\cite{murakami2018restricted} noted that a lower bound of $\frac{s^2}{\alpha^2 \eps^2}$ (the first term in~\eqref{eqn:hl_sample}) is immediately implied by previously known results on distribution estimation under standard LDP (\eqref{eqn:optimal-ldp} for $k=s$). However, obtaining a lower bound equalling the second term was still open. Using the recently proposed technique in~\cite{pmlr-v99-acharya19a}, we prove this lower bound in Section~\ref{sec:hl_lower_proof}.

\subsection{Achievability using a variant of HR} \label{sec:hl_upper}
We propose an algorithm based on Hadamard response~\cite{AcharyaSZ19}, which gives us a tight upper bound for $k$-ary distribution estimation under $(A, \varepsilon)$-high-low LDP. 
Let $s = |A|$ and $S$ be the smallest power of 2 larger than $s$, \textit{i.e.,} $S := 2^{\ceil{\log (s+1)}}$. Let $t := k - s$ be the number of all non-sensitive elements. Then we have $S + t \le 2(s+t) = 2 k$. Let $H_S$ be the $S \times S$ Hadamard matrix using Sylvester's construction. Define the output alphabet to be $[S+t] = \{1,...,S+t\}$. then the channel is defined as the following:
When $x \in A = [s]$, we have
\begin{align} \label{eqn:q_hl_1}
	Q(y \mid x) = 
	\begin{cases}
		\frac{2 e^\eps}{S(e^\eps + 1)} &\text{ if } y \in [S] \text{ s.t. } H_S(x, y) = 1, \\
		\frac{2}{S(e^\eps + 1)} &\text{ if } y \in [S] \text{ s.t. } H_S(x, y) = -1, \\
		0, &\text{ if } y \notin [S].
	\end{cases}
\end{align}
Else if $x \notin A$, we have
\begin{align} \label{eqn:q_hl_2}
	Q(y \mid x) = 
	\begin{cases}
		\frac{2}{S(e^\eps + 1)}, &\text{ if } y \in [S],\\
		\frac{e^\eps - 1}{e^\eps + 1}, &\text{ if } y = x + S -s,\\
		0, &\text{otherwise.}
	\end{cases}
\end{align}
It is easy to verify that this scheme satisfies $(A, \varepsilon)$-high-low LDP. 
\new{Next we construct estimators based on the observations $Y_1, Y_2 \upto Y_n$ and states its performance in Proposition~\ref{prop:upper_hlldp}. For all $i \in [s]$,
\begin{align} \label{eqn:ps_est}
	\widehat{\p_i} = \frac{2(e^\eps + 1)}{e^\eps - 1} \Paren{ \widehat{\p(S_i)} -\frac{1}{e^\eps + 1} } - \widehat{\p(A)},
\end{align}
where
    \begin{gather}
        \widehat{\p(A)} =  \frac{e^\eps + 1}{e^\eps - 1}  \Paren{\frac{1}{n} \sum_{m = 1}^{n} \mathbf{1}\{Y_m \in S\}- \frac{2}{e^\eps + 1}}, \nonumber\\
    	\widehat{\p(S_i)} =  \frac{1}{n} \left(\sum_{m = 1}^{n} \mathbf{1}\{Y_m \in S_i\}\right). \nonumber
    \end{gather}
For all $i \notin [s]$, we simply use the empirical estimates
\begin{equation} \label{eqn:pns_est}
	\widehat{\p_i}  = \frac{e^\eps + 1}{e^\eps - 1} \frac{1}{n} \left(\sum_{m = 1}^{n} \mathbf{1}\{Y_m = i + S -s\}\right).
\end{equation}
}
\begin{proposition}
    The estimators defined in~\eqref{eqn:ps_est} and~\eqref{eqn:pns_est} satisfy the following:
    \begin{equation} \label{eqn:var_bound_hl}
	    \expectation{\dtv{\hat{p}}{p} } \le \sqrt{\frac{3s^2}{n} \Paren{\frac{e^\eps + 1}{e^\eps - 1} }^2 } + \sqrt{\frac{e^\eps + 1}{e^\eps - 1} \frac{k}{n}}.
\end{equation}
\label{prop:upper_hlldp}
\end{proposition}

Since $\frac{e^\eps+1}{e^\eps-1} = O\Paren{\frac1\eps}$ when $\eps = O(1)$, setting the right hand side to be smaller than $\alpha$ gives us the upper bound part of Theorem~\ref{thm:hl_complexity}. For the proof of~\eqref{eqn:var_bound_hl}, please see Section~\ref{sec:hl_upper_proof}. In this scheme, each user only needs to communicate at most $\log k + 1$ bits since $S+t \le 2k$ while the scheme proposed in~\cite{murakami2018restricted} needs $\Omega(k)$ bits.

\subsection{Lower bound} \label{sec:hl_lower_proof}
We now prove the lower bound part of Theorem~\ref{thm:hl_complexity}.
A lower bound of $\Omega(s^2/\eps^2\alpha^2)$ follows from~\eqref{eqn:optimal-ldp}, which follows directly from the lower bounds on sample complexity of distribution under standard $\eps$-LDP (e.g., Theorem IV.1 in~\cite{YeB17}).

To prove a lower bound equalling the second term, we use the framework developed in~\cite{pmlr-v99-acharya19a} to prove lower bounds for distributed inference under various information constraints (e.g., privacy and communication constraints). Their key idea is to design a packing of distributions around the uniform distribution and show that the amount of information that can be gleaned from the output of these schemes is insufficient for distribution estimation. In particular, we will use their following result.

\begin{lemma}{[Lemma 13]\cite{pmlr-v99-acharya19a}} \label{lem:lower_bound}
	Let $\mathbf{u}$ be the uniform distribution over $[k]$ and $\cP$ be a family of distributions satisfying the following two conditions.
	\begin{enumerate}
		\item $\forall \mathbf{p} \in \cP$, we have $ \dtv{\mathbf{p}}{\mathbf{u}} \ge \dist. $
		\item $\forall \mathbf{p}_1 \in \cP$, we have
		\[
		|\{\mathbf{p}_2 \in \cP| \dtv{\mathbf{p}_1}{\mathbf{p}_2} \le \frac{\dist}{3}\}| \le C_\dist.
		\]
	\end{enumerate}
Suppose $\cQ$ is the set of all channels we can use to get information about $X$, then we have the sample complexity of $k$-ary distribution learning up to TV distance $\pm \dist/3$ under channel constraints $\cQ$ is at least
\[
\Omega\Paren{ \frac{\log |\cP| - \log C_\dist}{\max_{Q \in \cQ} \chi^2(Q| \cP)}},
\]
where $\mathbf{p}^Q$($\mathbf{u}^Q)$ is the distribution of $Y$ when $X \sim \mathbf{p}$ ($\mathbf{u}$), and
\begin{align}
\chi^2(\cP) := \frac{1}{|\cP|} \sum_{\mathbf{p} \in \cP} d_{\chi^2} (\mathbf{p}^Q, \mathbf{u}^Q), \nonumber \\ d_{\chi^2} (\mathbf{p}, \mathbf{q}) := \sum_{y \in \cY} \frac{(\mathbf{p}(y) - \mathbf{q}(y))^2}{\mathbf{q}(y)}. \nonumber
\end{align}
\end{lemma}
Let $k' = k - s \ge k/2$ be the number of non-sensitive elements. Let $\cZ = \{+1, -1\}^{k'}$ be the set of $k'$ bits. For all $z \in \cZ$, define $\p_z$ as the following
	\[
	\p_z(i) = \begin{cases}
	\frac1k + \frac{\dist \sum_{j=1}^{k'} z_j}{k'} &i = 1,\\
	\frac{1}{k} &i = 2, 3, ..., s,\\
	\frac1k + \frac{\dist z_{i - s}}{k'} & i = s+1, s+2, ..., k.\\
	\end{cases}
	\]
Let $\cP_\cZ = \{\p_z| z \in \cZ\}$ be the set of all distributions defined by $z \in \cZ$. Let $U_\cZ$ be a uniform distribution over set $\cZ$. Then we have
	\begin{align*}
		& \chi^2(Q | \cP_\cZ) \\ = & \EE_{z \sim U_\cZ} [ \sum_{y \in \cY} \frac{(\p^Q_z(y) - u^Q(y))^2}{u^Q(y)} ]  \\ 
		= &  \frac{k\dist^2}{k'^2} \EE_{z \sim U_\cZ} [ \sum_{y \in \cY} \frac{( \sum_{i = 1}^{k'}(Q(y|1)-Q(y|s+i))z_i )^2 }{\sum_{j \in[k]} Q(y|j)} ]  \\
		\le &  \frac{4\dist^2}{k}  \sum_{y \in \cY} \frac{ \sum_{i = 1}^{k'}(Q(y|1)-Q(y|s+i))^2 }{\sum_{j \in[k]} Q(y|j)} .
	\end{align*}
To bound this quantity, we have the following claim, which we prove in Section~\ref{sec:proof_chi_bound}. 
	\begin{claim} \label{clm:norm_bound}
		If $\forall\ Q \in \cQ$, we have: $\forall i \in \{s+1, s+2, ..., s+k'\}, y \in \cY$,
		\[
		Q(y|1) \le e^\eps Q(y|i),
		\]
	\begin{align}
	\sum_{i=1}^{k'} \sum_{y\in\cY} \frac{(Q(y|s+i)-Q(y|1))^2}{\sum_{j\in[\ab]}Q(y|j)}= O( \eps \ab).
	\end{align}
	\end{claim}
Moreover, we have $|\cP| = 2^{k'}$ and
\[
C_\alpha \le 2^{(1 - h(1/3)k')},
\]
where $h(x) = x \log(1/x) + (1-x)\log (1/(1-x))$. Combining these results and using Lemma~\ref{lem:lower_bound}, we get the sample complexity is at least
\[
	\Omega\Paren{ \frac{\log |\cP| - \log C_\dist}{\max_{Q \in \cQ} \chi^2(Q | \cP)}} = \Omega\Paren{ \frac{k}{\dist^2\eps}} .
\]

%% file: block.tex
\section{Distribution Estimation under BSLDP} 
\label{sec:bsldp}

For distribution estimation under block-structured LDP constraints we prove the following theorem.

\begin{theorem} \label{thm:bs_complexity}
Let $\eps=O(1)$, and $\partition = \{ \cX_1, \cX_2, ..., \cX_m \}$ be a partition of $\cX$ and $\cQ_{\partition, \eps}$ be the set of all possible channels that satisfy $(\partition, \eps)$-BSLDP, then
	\[
		n(k,\alpha,\cQ_{\mathbb{P},\eps}) = \Theta \Paren{\frac{\sum_{i=1}^m k_i^2}{\alpha^2 \eps^2}},
	\]
	where $\forall i \in [m], |\cX_i| = k_i$ and $|\cX| = k = \sum_{i=1}^{m} k_i$.
\end{theorem}
In the special case when all the blocks have the same size $k/m$, the sample complexity is $\Theta(k^2/(m\alpha^2\eps^2))$, which saves a factor of $m$ compared classic LDP. In Section~\ref{sec:bs_upper}, we describe an algorithm based on HR that achieves this error. Moreover, it only uses $O(\log k)$ bits of communication from each user. A matching lower is proved bound in Section~\ref{sec:bs_lower_proof}, which shows that our algorithm is information theoretically optimal.

\subsection{Achievablity using a variant of  HR} \label{sec:bs_upper}
The idea of the algorithm is to perform Hadamard Response proposed in~\cite{AcharyaSZ19} within each block. Without loss of generality we assume each block $\cX_j = \{ (j,i) \mid  i \in [k_j]\}$. For each block $\cX_j, j\in [m]$, we associate a Hadamard matrix $H_{K_j}$ with $K_j = 2^{\ceil{\log (k_j+1)}} $. Let $\cY_j = \{(j, i) \mid  i \in [K_j]\}$. 
For each $x = (j,i) \in \cX_j$, we assign the $(i+1)$th row of $H_{K_j}$ to $x$. Define the set of locations of `$+1$'s at the $(i+1)$th row of $H_{K_j}$ to be $S_x$. Then the output domain is $\cY:=\cup_{j =1}^m \cY_j$. The privatization scheme is given as
\[
	Q(Y = (j,i)\mid  X) = 
	\begin{cases}
		\frac{2e^\eps }{K_j (1 + e^\eps)},& X \in \cX_j, i \in S_X, \\
		\frac{2}{K_j (1 + e^\eps)},& X \in \cX_j, i \notin S_X, \\
		0, &\text{elsewhere}.
	\end{cases} 
\]
It is easy to verify that this scheme satisfies the privacy constraints. 
\new{Next we construct estimators based on $Y_1, Y_2 \upto Y_n$ and state the performance in Proposition~\ref{prop:upper_bsldp}. Let $Y(1), Y(2)$ be the two coordinates of each output $Y$. For each $j \in [m]$ and $x \in \cX_j$, 
\begin{equation} \label{eqn:pi_est}
	\hat{\p_x} = \frac{2(e^\eps + 1)} {e^\eps - 1} \Paren{	\widehat{\p(S_x)}  - \frac{\widehat{\p(\cX_j)} }{2}},
\end{equation}
where
\begin{gather*}	
\widehat{\p(\cX_j)}  = \frac{1}{n} \sum_{t = 1}^{n} \mathbf{1}\{ Y_t(1) = j\}, \nonumber \\
\widehat{\p(S_x)} = \frac{1}{n} \sum_{t = 1}^{n} \mathbf{1}\{ Y_t(1) = j, Y_t(2) \in S_x\}. \nonumber
\end{gather*}
}
\begin{proposition} \label{prop:upper_bsldp}
	 Under the unbiased estimator $\hat{p}$ described in~\eqref{eqn:pi_est}, we have:
	\begin{align} \label{eqn:var_bound_bs}
		\expectation{ \ell_2^2(\hat{p}, \p)} \le \frac{12 \max_i k_i}{n} \Paren{\frac{e^\eps + 1} {e^\eps - 1} }^2, \nonumber \\ \expectation{\dtv{\hat{p}}{p}} \le  \frac{2(e^\eps + 1)} {e^\eps - 1}  \sqrt{ \frac{3\sum_{j = 1}^{m}  k_i^2}{n}  }.
	\end{align}
\end{proposition}

Since $\frac{e^\eps+1}{e^\eps-1} = O\Paren{\frac1\eps}$ when $\eps = O(1)$, we get desired bounds in Theorem~\ref{thm:bs_complexity}. For the proof of~\eqref{eqn:var_bound_bs}, see Section~\ref{sec:block_upper_proof}. We note here that our algorithm also gives the optimal bound in terms of $\ell_2$ distance. A matching lower bound can shown using well established results~\cite{duchi2013local} on LDP by considering the maximum of expected loss if we put all the mass on each single block.
\subsection{Lower bound} \label{sec:bs_lower_proof}

We now prove the lower bound part of Theorem~\ref{thm:bs_complexity}. The general idea is similar to the proof in Section~\ref{sec:hl_lower_proof}, which is based on Lemma~\ref{lem:lower_bound}. Without loss of generality, we assume all the $k_i$'s are even numbers \footnote{\new{If one of the $k_i$'s is odd, we can remove one element from $\cX_j$, which will make the problem simpler and the sample complexity remains unchanged up to constant factors}}. We construct a family of distributions as following: Let $\cZ = \cZ_1 \times \cZ_2 \times \cdots \times \cZ_m$ and $\forall j \in [m], \cZ_j = \{+1, -1\}^\frac{k_j}{2}$. $\forall z \in \cZ$, we denote the $j$th entry of $z$ as $z_j$ where $z_j \in \cZ_j$. Define  $z_{j,i}$ to be the $i$th bit of $z_j$. $\forall j \in [m]$ and $i \in [k_j/2]$, we have
\vspace{-5pt}
\begin{gather*}
	p_z((j, 2i-1)) = \frac{1}{k} + \frac{2 z_{j,i} k_j \alpha}{\sum_{t = 1}^{m}k_i^2}, \nonumber \\ p_z((j, 2i)) = \frac{1}{k} - \frac{2 z_{j,i} k_j \alpha}{\sum_{t = 1}^{m}k_i^2}.
\end{gather*}
Note that $\forall z \in \cZ$, 
$d_{\rm TV}({p_z}, u) = \alpha.$ Moreover, 
$|\cP| = 2^{\frac{k}{2}}$ and $|C_\alpha| \le 2^{\frac{k}{2} h(1/3)}$ where $h$ is the binary entropy function. By Lemma~\ref{lem:lower_bound}, let $\cQ$ be the set of channels that satisfy $(\mathbb{P}, \eps)$-LDP and $\cP_\cZ = \{\p_z\mid  z\in \cZ\}$, it would suffice to show that
\begin{equation} \label{eqn:bound_chi_bs}
    \max_{Q \in \cW} \chi^2(Q \mid \cP_\cZ) = O \Paren{\frac{k \alpha^2 \eps^2 }{\sum_{i = 1}^m k_i^2}}.
\end{equation}
The proof of~\eqref{eqn:bound_chi_bs} is technical and presented in Section~\ref{sec:bound_chi_bs}.

%% file: experiments.tex
\vspace{-3pt}
\section{Experiments} \label{sec:experiments}
We perform experiments on both synthetic data and real data to empirically validate how the new notion of context-aware LDP and associated algorithms would affect the accuracy of $k$-ary distribution estimation. Specifically, we choose the special case of block-structured LDP. 

For synthetic data, we set $k = 1000, \eps = 1$. We assume all the blocks to have the same size and $m \in \{10, 20 ,50 ,100\}$. For traditional LDP, we use Hadamard Response~\cite{AcharyaSZ19}, one of the state-of-the-art sample-optimal algorithms. We take $n = 1000 \times 2^i, i \in \{0, 1, \cdots, 9\} $ and generate samples from the following three distributions: Geometric distribution $Geo(\lambda)$ where $\p(i) \propto (1-\lambda)^i \lambda$, Zipf distribution $Zipf(\lambda)$ where $\p(i) \propto (i+1)^{-\lambda}$ and uniform distribution. We assume the blocks are partitioned based on their indices (the first block is $\{0, 1, \cdots, k/m -1\}$, the second is $\{k/m, k/m+1, \cdots, 2k/m -1\}$ and so on). For Geometric distribution, we consider another case where we permute the mass function over the symbols to spread the heavy elements into multiple blocks, denoted by $Geo(\lambda)^*$. The results are shown in Figure~\ref{fig:synthetic} and each point is the average of 10 repetitions of the same experiment. We can see that we get consistently better accuracy under the notion of block-structured LDP compared to the classical notion. Moreover, the larger $m$ we have, the better accuracy we get, which is consistent with our analysis.
\begin{table}[b!]
\centering
\begin{tabular}{lllll}
\hline
\multicolumn{1}{|c|}{$(m_1, m_2)$}   & \multicolumn{1}{c|}{LDP}   & \multicolumn{1}{c|}{(5,7)} & \multicolumn{1}{c|}{(25, 35)} & \multicolumn{1}{c|}{(25, 70)} \\ \hline
\multicolumn{1}{|c|}{$d_{TV}$-error} & \multicolumn{1}{c|}{0.591} & \multicolumn{1}{c|}{0.298} & \multicolumn{1}{c|}{0.108}    & \multicolumn{1}{c|}{0.082}    \\ \hline
\end{tabular}
\caption{$d_{TV}$ estimation error under different $(m_1, m_2)$ pairs.}
\label{tab:gowalla}
\end{table}

\begin{figure}
		\centering
		\subfigure[	{{\small Uniform}} ]{\includegraphics[width=0.4\textwidth]{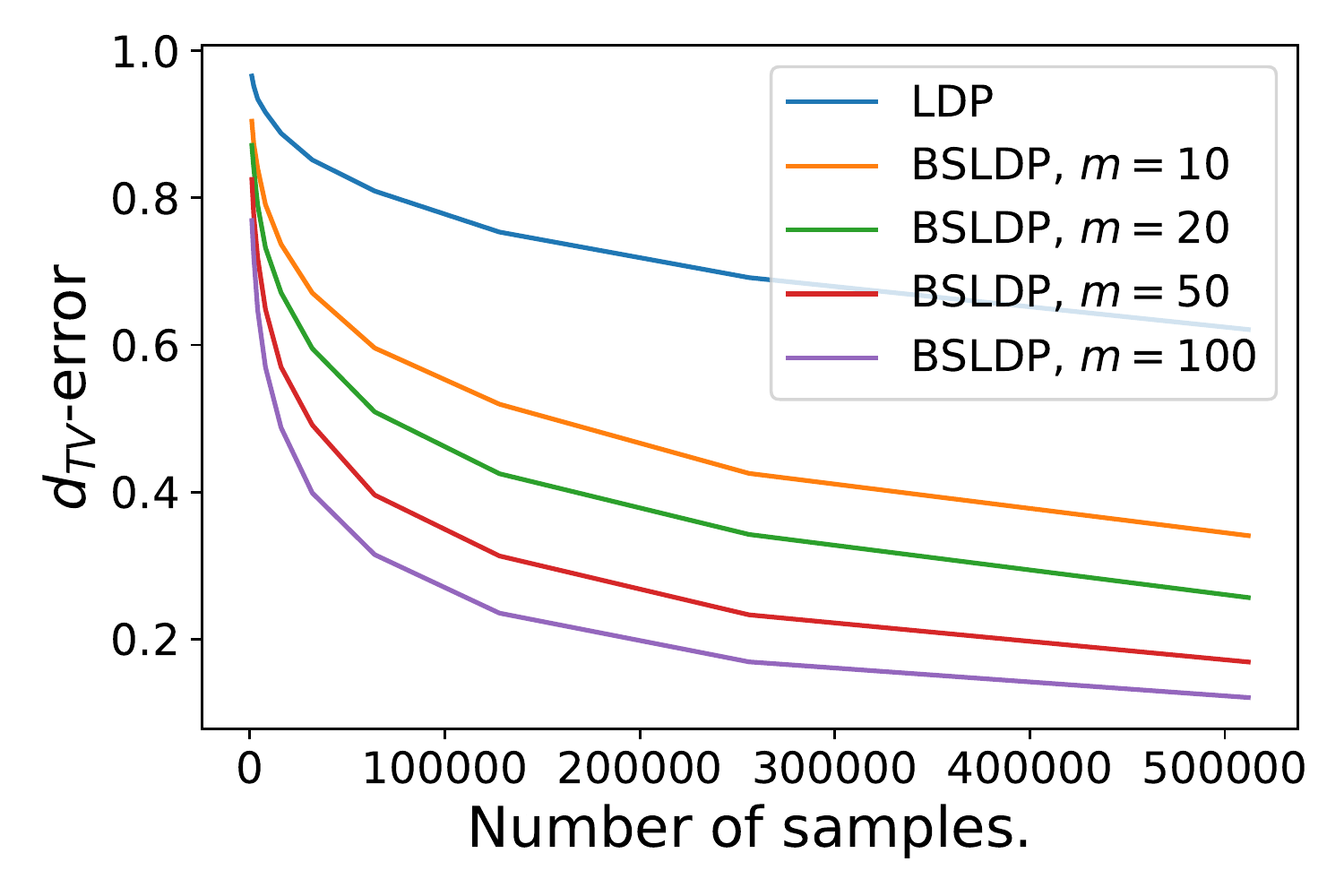}\label{fig:k_1000_eps_05}}
		\subfigure[	{{\small Geo(0.95)}}
		]{\includegraphics[width=0.4\textwidth]{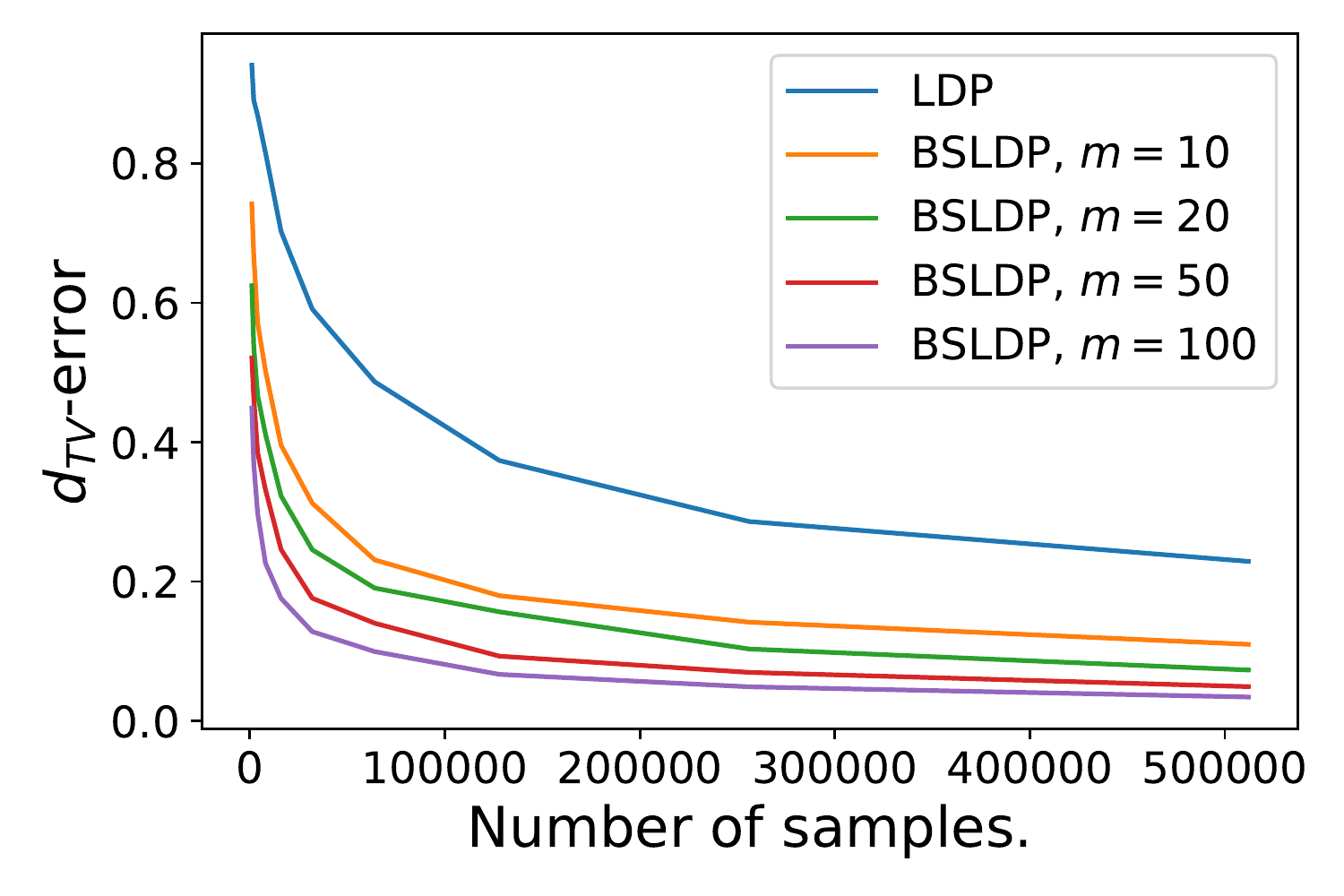}\label{fig:k_1000_eps_2}}
		\subfigure[	{{\small Zipf(1)}} ]{\includegraphics[width=0.4\textwidth]{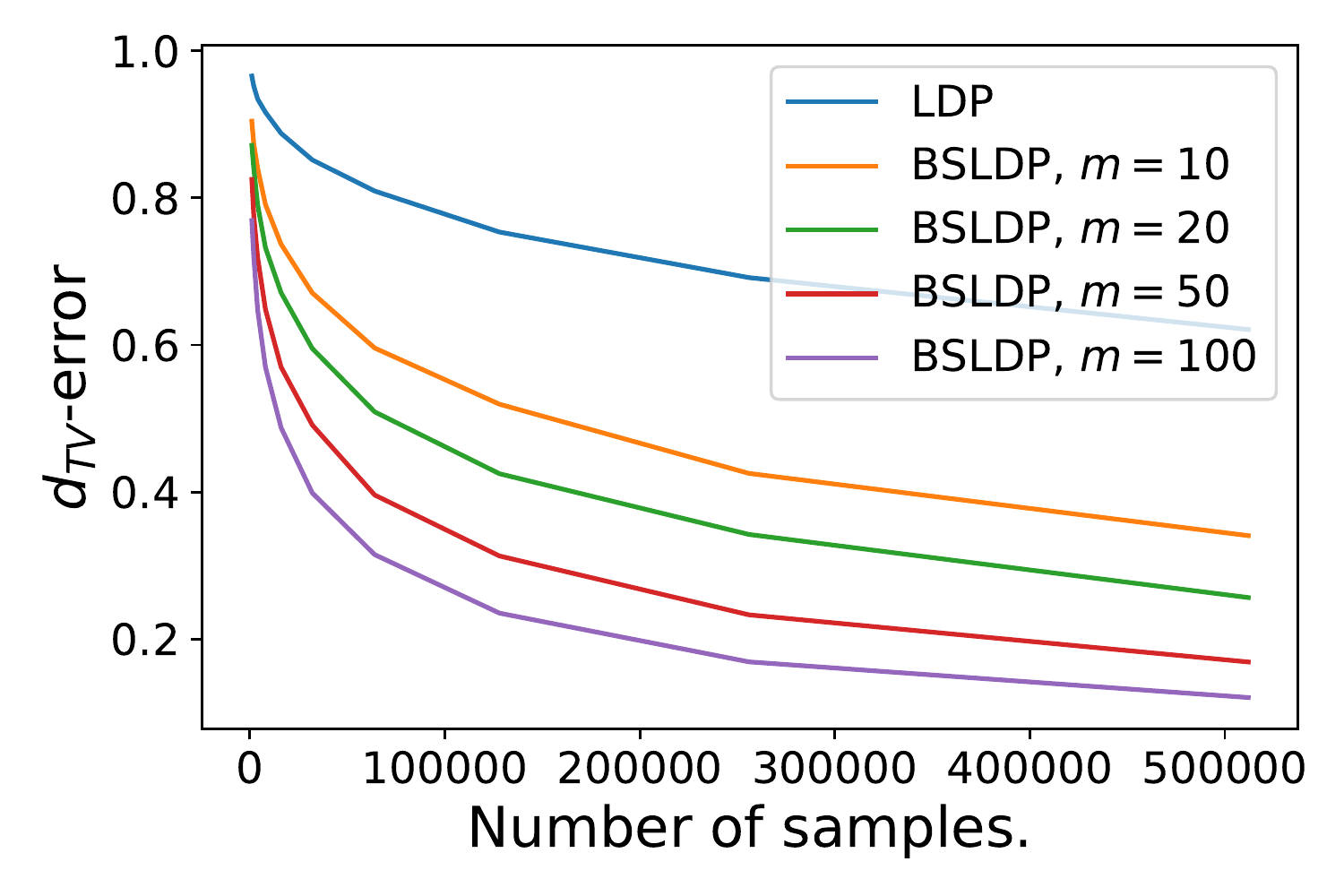}\label{fig:k_1000_eps_5}}
		\subfigure[	{{\small Geo(0.95)$^*$}} ]{\includegraphics[width=0.4\textwidth]{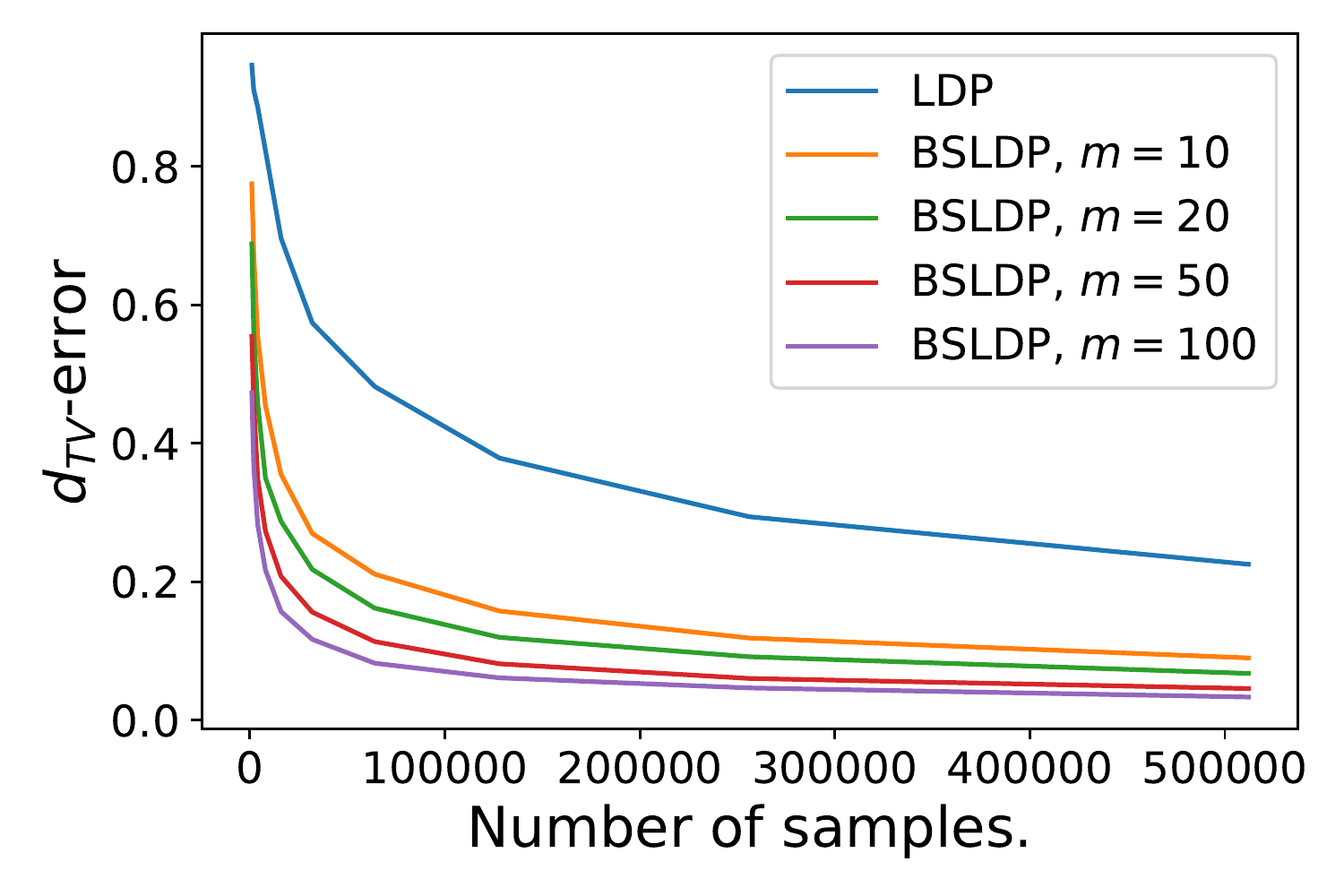}\label{fig:k_1000_eps_7}}
		\caption[ ]
		{\small $d_{TV}$-error comparison under different distributions}
		\label{fig:synthetic}
\end{figure}
To validate our algorithm on real datasets, we take the Gowalla user check-in dataset~\cite{cho2011friendship}, which consists of user check-in records with location information (latitudes and longitudes) on the Gowalla website. We take 3671812 check-in records with locations within 25N and 50N for latitude and 130W and 60W for longitude (mostly within continental US). We round the latitude and longitude for each record up to accuracy 0.2 and regard records with the same latitude and longitude as the same location. By doing so, we get a dataset with $43,750$ possible values. \new{Figure~\ref{fig:heatmap} shows the empirical distribution of the check-in records of the dataset.} We take the empirical distribution of the records as our ground truth and try to estimate it while preserving the privacy of each record. We partition latitudes into $m_1$ equal parts and longitudes into $m_2$ equal parts. The resulting grid will be used as the blocks ($m_1m_2$ blocks in total). Table~\ref{tab:gowalla} shows the average $d_{TV}$ error over 100 runs of the experiment for LDP and BS-LDP with different $(m_1, m_2)$ pairs. From the table we can see that by switching to block-structured LDP, we can get more meaningful estimation accuracy compared to classical LDP.
\begin{figure}
		\centering
	    \includegraphics[width=0.7\textwidth]{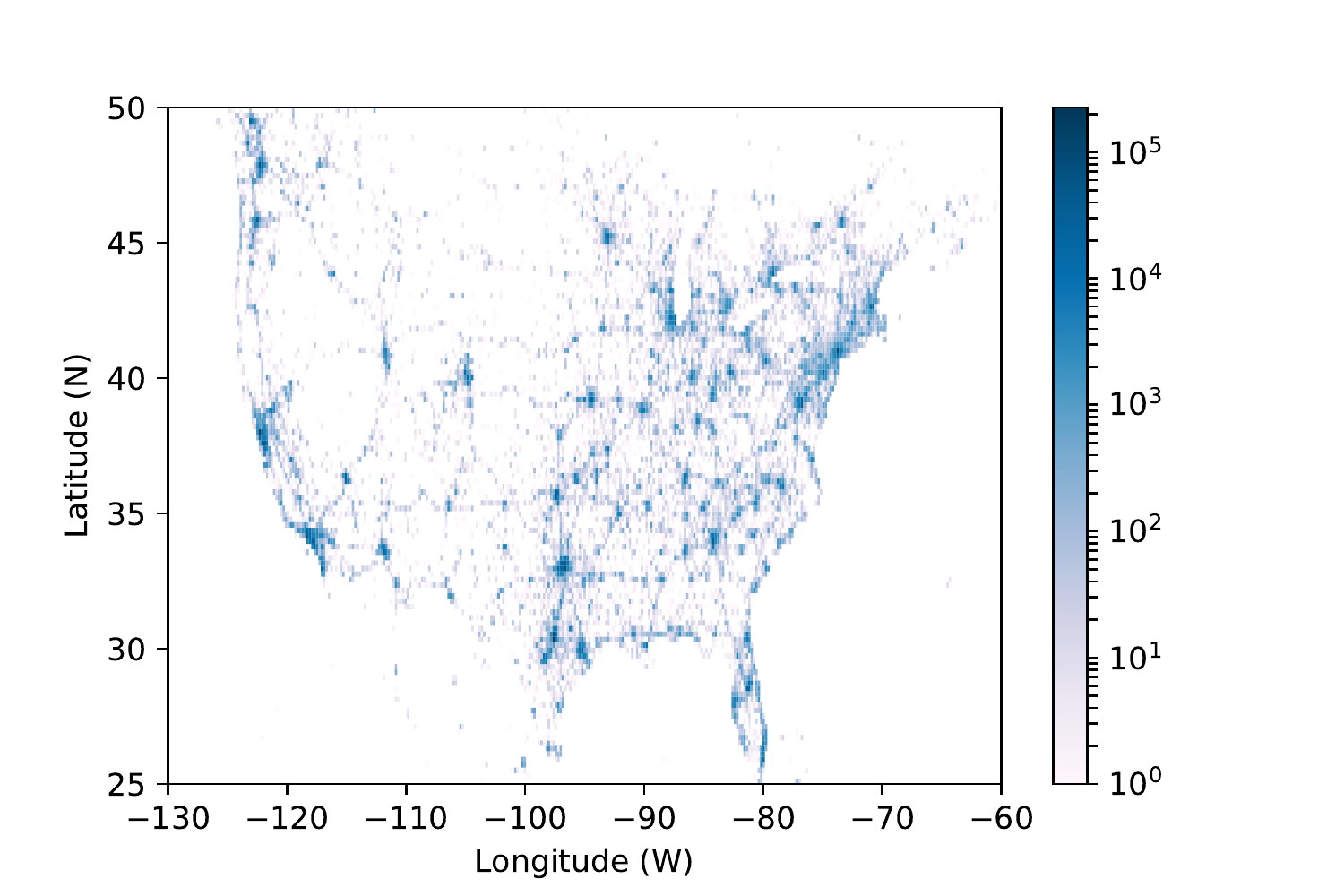}
		\caption[ ]
		{\small Heatmap (distribution) of the check-in records.}   
		\label{fig:heatmap}
\end{figure}

%% file: conclusion.tex
\vspace{-5pt}
\section{Conclusion}
\label{sec:conc}
We presented a unifying context-aware LDP framework and investigated communication, privacy, and utility trade-offs under both binary and $k$-ary data domains for minimax distribution estimation. Our theory and experiments on synthetic and real-world data suggest that context-aware LDP leads to substantial utility gains compared to vanilla LDP. In order to examine the effect of the number of data partitions in block-structured LDP, our experiments focused on uniform partitioning of geo-location data and examined the utility gains and various partition sizes. In practice however, non-uniform partitioning can better reflect the different topologies of cities. Thus, more careful experiments with non-uniform partitions need to performed to better quantify the utility gains. More broadly, more experiments should be conducted to verify that the gains we see on the Gowalla dataset are also applicable in other data domains.

%% file: proof_properties.tex
\section{Proof of Properties} \label{sec:proof_properties}
In this section, we prove several properties of Context-Aware LDP stated in Section~\ref{sec:context}.

\subsection{Proof of Proposition~\ref{prop:post-processing}} 
    $\forall x_1, x_2 \in \cX$ and $y \in \cY_2$, we have:
    \begin{align*}
       \frac{\probof{\cA(x_1) = y}}{\probof{\cA(x_2) = y}} = \frac{\sum_{y' \in \cY_1} \probof{\cA_1(x_1) = y'} \probof{\cA_2(y') = y} }{\sum_{y' \in \cY_1} \probof{\cA_1(x_2) = y'} \probof{\cA_2(y') = y}} 
    \le  \max_{y' \in \cY_1} \frac{\probof{\cA_1(x_1) = y'}}{\probof{\cA_1(x_2) = y'}} \le e^{\eps_{1,2}}
    \end{align*}

\subsection{Proof of Proposition~\ref{prop:adapt_compose}} 
    $\forall x_1, x_2 \in \cX$ and $y_1 \in \cY_1$, $y_2 \in \cY_2$, we have:
    \begin{align*}
       \frac{\probof{\cA_1(x_1) = y_1, \cA_2(x_1, y_1) = y_2}}{\probof{\cA_1(x_2) = y_1, \cA_2(x_2, y_1) = y_2}}
    = \frac{\probof{\cA_1(x_1) = y_1} \probof{\cA_2(x_1, y_1) = y_2}}{\probof{\cA_1(x_2) = y_1}\probof{\cA_2(x_2, y_1) = y_2}}
    \le  e^{\eps^{(1)}_{1,2}+\eps^{(2)}_{1,2}}
    \end{align*}

\subsection{Proof of Proposition~\ref{prop:side_information}} 
\[
    \frac{\probof{X = x_1| Y = y}}{\probof{X = x_2| Y = y}} = \frac{\probof{Y = y | X = x_1}\probof{X = x_1}}{\probof{Y = y | X = x_2}\probof{X = x_1}} \le e^{\eps_{1,2}} \frac{\p^*(x_1)}{\p^*(x_2)}.
\]
    
\subsection{Proof of Theorem~\ref{thm:operational_def_gdp}} \label{sec:operational_def_gdp}
Let $Q$ be an $\Epsilon$-LDP scheme from $\cX$ to $\cY$. Note that $\X-\Y-\hat{\X}$ form a Markov chain. 
\begin{align}
P_{\rm MD}(x,x') &= \probof{\hat{\X} = x' | \X = x} \nonumber \\
	& = \sum_{y \in \cY} \probof{\hat{\X} = x' | \Y = y} \probof{Y = y | \X = x}\nonumber \\
	& = \sum_{y \in \cY} \probof{\hat{\X} = x' | \Y = y} Q({y | x})\nonumber \\
	& \ge \sum_{y \in \cY} {\probof{\hat{\X} = x' | \Y = y} } Q({y | x'}) e^{-\eps_{x',x}} \nonumber \\
	& = e^{-\eps_{x',x}}  \probof{\hat{\X} = x' | \X = x'}  \nonumber\\
	& = e^{-\eps_{x',x}} (1 - P_{\rm FA}(x,x')). \nonumber 
\end{align}
Rearranging the terms gives~\eqref{eqn:famd}.~\eqref{eqn:mdfa} can be obtained similarly starting with $P_{\rm FA}(x,x')$.

Now for the other direction, consider a decision rule that satisfies~\eqref{eqn:famd} and~\eqref{eqn:mdfa}. For any $x, x' \in \cX$ and $S \subset \cY$, consider a decision rule that outputs $\hat{\X} = x$ if $Y \in S$ and $\hat{\X} = x'$ otherwise. Then we get,
\begin{align}
\frac{Q(S|x)}{Q(S|x')} = \frac{\probof{\hat{\X} = x | X = x}}{\probof{\hat{\X} = x | X = x'}} = \frac{1 - P_{\rm MD}(x,x')}{P_{\rm FA}(x,x')}, \nonumber
\end{align}
which is bounded by $e^{\eps_{x,x'}}$ according to~\eqref{eqn:mdfa}, showing that the scheme is $\Epsilon$-LDP.

%% file: binary_proof.tex
\section{Proof of Theorem~\ref{thm:optimal_binary}}

In Section~\ref{sec:operational}, we give the operational meaning of context-aware LDP in Theorem~\ref{thm:operational_def_gdp}. For binary alphabet, if $Y$ is generated using an $E$-LDP mechanism, the error region for all binary hypothesis testing rule $\hat{\X}: \cY \rightarrow \{1, 2\}$, denoted by $\cR_{\eps_{1,2}, \eps_{2,1}}$ can be expressed by the convex hull defined by the following three points:
\begin{align} \label{eqn:err_vertex_1}
    (P_{\rm FA}^1, P_{\rm MD}^1) = (1, 0), & &(P_{\rm FA}^2, P_{\rm MD}^2) = (0, 1), 
\end{align}
\begin{align}\label{eqn:err_vertex_2}
    (P_{\rm FA}^3, P_{\rm MD}^3) = (\frac{e^{\eps_{2,1}} -1}{e^{\eps_{1,2}+\eps_{2,1}}-1}, 
    \frac{e^{\eps_{2,1}} -1}{e^{\eps_{1,2}+\eps_{2,1}}-1}),
\end{align}
where $P_{\rm FA} = \probof{\hat{X} = 1 | X = 2}$ and $P_{\rm MD} = \probof{\hat{X} = 2 | X = 1}$.

Next we show if we use the scheme expressed in~\eqref{eqn:binary_optimal}, the error region $\cR_{Q^*} = \cR_{\eps_{1,2}, \eps_{2,1}}$. By Theorem~1, we know $\cR_{Q^*} \subset \cR_{\eps_{1,2}, \eps_{2,1}}$. Hence we only need to show the reverse direction. More specifically, we only need to show the three vertices expressed in~\eqref{eqn:err_vertex_1} and~\eqref{eqn:err_vertex_2} can be achieved. The two vertices in~\eqref{eqn:err_vertex_1} can be achieved by trivial rules $\hat{X} = 1$ and $\hat{X} = 2$. Next we show the decision rule
\[
    \hat{X}(Y) = Y.
\]
achieves the error in~\eqref{eqn:err_vertex_2}. Using this decision rule, we have:
\[
    P_{\rm FA} = \probof{\hat{X} = 1 | X = 2} = \frac{e^{-\varepsilon_{1,2}}\left(e^{\varepsilon_{2,1}}-1\right)}{e^{\varepsilon_{2,1}}-e^{-\varepsilon_{1,2}}} = \frac{e^{\eps_{2,1}} -1}{e^{\eps_{1,2}+\eps_{2,1}}-1},
\]
\[
    P_{\rm MD} = \probof{\hat{X} = 2 | X = 1} = \frac{1-e^{-\varepsilon_{1,2}}}{e^{\varepsilon_{2,1}}-e^{-\varepsilon_{1,2}}} = \frac{e^{\eps_{1,2}} -1}{e^{\eps_{1,2}+\eps_{2,1}}-1},
\]
which completes the proof.

For any other scheme $Q \in \cQ_\eps$, we have $\cR_Q \subset \cR_{\eps_{1,2}, \eps_{2,1}} = \cR_{Q^*}$. Hence by Theorem 20 in~\cite{kairouz2014extremal}, we have if $Y_Q$ and $Y_{Q^*}$ are the outputs of schemes $Q$ and $Q^*$ respectively, there exists a coupling between $Y_Q$ and $Y_{Q^*}$ such that $\X - \Y_{Q^*} - \Y_{Q}$ forms a Markov chain, which, by data processing inequality, implies
\[
    U(Q) \le U(Q^*).
\]

%% file: upper_proof.tex
\section{Upper Bound Proofs}
\subsection{Estimator for the high-low model} \label{sec:hl_estimator}
In this section, for the high-low model, we show how the output distribution is related to the input distribution and construct an estimator based on them. For all $i \in [s]$, define set $S_i = \{y | y \in [S], H_S(i+1,y) = +1\}$. Then using properties of Hadamard matrices in Section~\ref{sec:preliminaries},

\begin{align}
	\p([S]) &:= \frac{2}{e^\eps + 1}  + \frac{e^\eps - 1}{e^\eps + 1} \probof{x \in A}.\label{eqn:pS} \\
	\p(S_i) &:=  \probof{y \in S_i} = \sum_{x \in [k]} \probof{Y \in S_i | X = x}  p_x\nonumber\\
	&=p_i |S_i| \frac{2 e^\eps}{S(e^\eps + 1)} + \sum_{x \in A, x \neq i} p_x\Paren{|S_i \cap S_x| \frac{2e^\eps}{S(e^\eps + 1)} + |S_i \cap S_x^c| \frac{2}{S(e^\eps + 1)}} + \sum_{x \notin A} p_x |S_i| \frac{2}{S(e^\eps + 1)}  \nonumber \\
	&=\frac{1}{e^\eps + 1} + \frac{e^\eps - 1}{2(e^\eps + 1)} \probof{x \in A} +  \frac{e^\eps - 1}{2(e^\eps + 1)} p_i.  \label{eqn:pAi} 
\end{align}

Once we observe $Y_1, Y_2 \upto Y_n$, we can get the following unbiased empirical estimates: for all $i \in [s]$,
\begin{align}
	\widehat{\p(S_i)} =  \frac{1}{n} \left(\sum_{m = 1}^{n} \mathbf{1}\{Y_m \in S_i\}\right), \widehat{\p([S])} = \frac{1}{n} \left(\sum_{m = 1}^{n} \mathbf{1}\{Y_m \in S\}\right). \nonumber
\end{align}
Our estimates for these $\p_i$'s will be
\begin{align*} 
	\widehat{\p_i} = \frac{2(e^\eps + 1)}{e^\eps - 1} \Paren{ \widehat{\p(S_i)} -\frac{1}{e^\eps + 1} } - \widehat{\p(A)},  \text{ where } \widehat{\p(A)} =  \frac{e^\eps + 1}{e^\eps - 1}  \Paren{\widehat{\p([S])} - \frac{2}{e^\eps + 1}}.
\end{align*}

For all $i \notin [s]$, we simply use the empirical estimates
\begin{equation*} 
	\widehat{\p_i}  = \frac{e^\eps + 1}{e^\eps - 1} \frac{1}{n} \left(\sum_{m = 1}^{n} \mathbf{1}\{Y_m = i + S -s\}\right).
\end{equation*}

\subsection{Error bound proof for the high-low model (Proposition~\ref{prop:upper_hlldp})} \label{sec:hl_upper_proof}
In this section, we bound both the expected $\ell_1$ risk and $\ell_2$ risk of the estimator proposed in~\eqref{eqn:ps_est} and~\eqref{eqn:pns_est}.
\begin{align}
	\expectation{\ell_2^2(\hat{p}, p) }& = \sum_{i= 1}^{k}\variance{\widehat{\p_i}} = \sum_{i= 1}^{s}\variance{\widehat{\p_i}} + \sum_{i= s+1}^{k}\variance{\widehat{\p_i}} \nonumber \\
	& \le \sum_{i= 1}^{s} [2 \Paren{\frac{2(e^\eps + 1)}{e^\eps - 1} }^2 \variance{\widehat{\p(S_i)}} + 2 \variance{\widehat{\p(A)}} ]+  \sum_{i= s+1}^{k}\variance{\widehat{\p_i}} \nonumber \\
	& = \sum_{i= 1}^{s} [2 \Paren{\frac{2(e^\eps + 1)}{e^\eps - 1} }^2 \frac{\p(S_i)(1-\p(S_i))}{n} + 2 \frac{\p(A)(1-\p(A))}{n} ]+  \sum_{i= s+1}^{k}\variance{\widehat{\p_i}} \nonumber \\
	& \le \frac{1}{n} \Paren{ \sum_{i= 1}^{s}  3 \Paren{\frac{e^\eps + 1}{e^\eps - 1} }^2 + \sum_{i= s+1}^{k} \Paren{\frac{e^\eps + 1}{e^\eps - 1}}^2 \frac{e^\eps - 1}{e^\eps + 1} \p_i(1- \frac{e^\eps - 1}{e^\eps + 1} \p_i)} \nonumber \\
	& \le \frac{1}{n} \Paren{ 3s \Paren{\frac{e^\eps + 1}{e^\eps - 1} }^2  + \frac{e^\eps + 1}{e^\eps - 1}}.
\end{align}

Similarly, we get
\begin{align}
	\expectation{\ell_1(\hat{p}, p) }& \le \sqrt{s\sum_{i= 1}^{s}\variance{\widehat{\p_i}} }+ \sqrt{k\sum_{i= s+1}^{k}\variance{\widehat{\p_i}}} \nonumber \\
	& \le \sqrt{\frac{3s^2}{n} \Paren{\frac{e^\eps + 1}{e^\eps - 1} }^2 } + \sqrt{\frac{e^\eps + 1}{e^\eps - 1} \frac{k}{n}}.
\end{align}

\subsection{Estimator for the block-structured model} \label{sec:bs_estimator}
In this section, for the block-structured model, we show how the output distribution is related to the input distribution and construct an estimator based on them. Let $Y(1), Y(2)$ be the two coordinates of each output $Y$. Then for each block $\cX_j$ and $x \in \cX_j$, define
\begin{align} 
\p(\cX_j) := \probof{X \in \cX_j}, && \p(S_x) := \probof{Y(1) = j, Y(2) \in S_x}. \nonumber
\end{align}
Using properties of Hadamard matrices in Section~\ref{sec:preliminaries}, $\forall j \in [m]$ and $x \in \cX_j$,
\begin{align} \label{eqn:pSi}
	\p(S_x) &= \probof{X \in \cX_j, X \neq x} \Paren{|S_x \cap S_X| \frac{2e^\eps }{K_j (1 + e^\eps)} + |S_x \cap S_X^c| \frac{2}{K_j (1 + e^\eps)}} +  \p_x \frac{2e^\eps }{K_j (1 + e^\eps)} |S_x| \nonumber \\ 
	& = \frac{\p(\cX_j)}{2} + \frac{e^\eps - 1}{2(e^\eps + 1)} \p_x
\end{align}
By observing $Y_1, Y_2 \upto Y_n$, we obtain empirical estimates for $\p(\cX_j)$ and $\p(S_x)$ as following. For each $j \in [m]$ and $x \in \cX_j$, 
\begin{gather*}	
\widehat{\p(\cX_j)}  = \frac{1}{n} \sum_{t = 1}^{n} \mathbf{1}\{ Y_t(1) = j\}, \nonumber \\
\widehat{\p(S_x)} = \frac{1}{n} \sum_{t = 1}^{n} \mathbf{1}\{ Y_t(1) = j, Y_t(2) \in S_x\}. \nonumber
\end{gather*}
Then from~\eqref{eqn:pSi}, 
\begin{equation} 
	\hat{\p_x} = \frac{2(e^\eps + 1)} {e^\eps - 1} \Paren{	\widehat{\p(S_x)}  - \frac{\widehat{\p(\cX_j)} }{2}} \nonumber
\end{equation}
is an unbiased estimate for $\p_x$'s.

\subsection{Error bound proof for the block-structured model (Proposeition~\ref{prop:upper_bsldp})} \label{sec:block_upper_proof}
In this section, we bound both the expected $\ell_1$ risk and $\ell_2$ risk of the estimator proposed in~\eqref{eqn:pi_est}.
\begin{align}
	\expectation{ \ell_2^2(\hat{p}, \p)} & = \sum_{x \in \cX} \variance{\hat{\p_x} } = \sum_{j = 1}^{m} \sum_{x \in \cX_j} \variance{\hat{\p_x} } \nonumber \\
	& \le  \sum_{j = 1}^{m} \sum_{x \in \cX_j} \Paren{\frac{2(e^\eps + 1)} {e^\eps - 1} }^2 \Paren{ 2 \variance{ \widehat{\p(S_x)} } + \frac{1}{2} \variance{ \widehat{\p(\cX_j)}} } \nonumber \\
	& = \Paren{\frac{2(e^\eps + 1)} {e^\eps - 1} }^2 \sum_{j = 1}^{m} \sum_{x \in \cX_j}  \frac{1}{n} \Paren{ 2\p(S_x) (1- \p(S_x)) + \frac{1}{2} \p(\cX_j) (1 - \p(\cX_j))} \nonumber \\
	& \le \Paren{\frac{2(e^\eps + 1)} {e^\eps - 1} }^2 \sum_{j = 1}^{m} \sum_{x \in \cX_j}  \frac{3}{n} \p(\cX_j) \nonumber \\
	& \le \frac{12}{n} \Paren{\frac{e^\eps + 1} {e^\eps - 1} }^2 \sum_{j = 1}^{m} k_j\p(\cX_j) \nonumber \\
	& \le \frac{12 \max_j k_j}{n} \Paren{\frac{e^\eps + 1} {e^\eps - 1} }^2.   \nonumber 
\end{align}

For $\ell_1$ error, using similar steps, we get:
\begin{align}
	\expectation{\ell_1(p, \hat{p})} & = \sum_{j = 1}^{m} \sum_{x \in \cX_j} |\p_x - \hat{\p_x}| \le  \sum_{j = 1}^{m} \sqrt{ k_j \sum_{x \in \cX_j}\variance{\hat{\p_x} }} \le \sum_{j = 1}^{m} \frac{2(e^\eps + 1)} {e^\eps - 1}  \sqrt{ \frac{3k_j^2}{n} \p(\cX_j)  }  \nonumber \\
&\le  \frac{2(e^\eps + 1)} {e^\eps - 1}  \sqrt{ \frac{3\sum_{j = 1}^{m}  k_j^2}{n} \sum_{j = 1}^{m}\p(\cX_j)} \nonumber \\
& = \frac{2(e^\eps + 1)} {e^\eps - 1}  \sqrt{ \frac{3\sum_{j = 1}^{m}  k_j^2}{n} }.
\end{align}
The second last inequality comes from Cauchy-Schwarz inequality.

%% file: lower_bound.tex
\section{Lower Bound Proofs}
\subsection{Proof of Claim~\ref{clm:norm_bound}} \label{sec:proof_chi_bound}
By definition, we have
\begin{align}
	& \sum_{y\in\cY} \sum_{i=1}^{k'} \frac{(Q(y|s + i)-Q(y|1))^2}{\sum_{1\le i'\le \ab}Q(y|i')}  \le \sum_{y\in\cY} \sum_{i=1}^{k'} \frac{(Q(y|s + i)-Q(y|1))^2}{\sum_{1\le i'\le \ab' }Q(y|s + i') + Q(y|1)}  \nonumber \\
	= & \sum_{y\in\cY} \Paren{ \frac{\sum_{i=1}^{k'} Q(y|s+i)^2 + Q(y|1)^2}{\sum_{1\le i'\le \ab' }Q(y|s + i') + Q(y|1)} + \frac{\sum_{i=1}^{k'}  Q(y|1) (Q(y|1)- Q(y|s+i) )}{\sum_{1\le i'\le \ab' }Q(y|s + i') + Q(y|1)} - Q(y|1) } \nonumber \\
	\le &\sum_{y\in\cY}  \Paren{ \max_i Q(y|i)  +\frac{ \sum_{i=2}^{k'}  Q(y|1) (e^{\eps}-1) Q(y|s + i) }{\sum_{\le i'\le \ab'}Q(y|s + i') + Q(y|1)} - Q(y|1) } \nonumber \\ 
	\le &\sum_{y\in\cY} \Paren{ \max_i Q(y|i) + Q(y|1) (e^{\eps}-1) - Q(y|1) }. \label{eqn:three}
\end{align}

To proceed, we need the following lemma.
\begin{lemma} \label{lem:subset}
	For all set $M \subset \cY$, $\forall i \in [k]$, we have:
	\[
	\sum_{y \in M} Q(y|i) \le 1 - e^{-\eps} \Paren{1- \sum_{y \in M} Q(y|1)}.
	\]
\end{lemma}

\begin{proof}
	\begin{align*}
		\sum_{y \in M} Q(y|i)  = 1 - \sum_{y \in M^c} Q(y|i) \le 1 - e^{-\eps} \sum_{y \in M^c} Q(y|1) =  1 - e^{-\eps} \Paren{1- \sum_{y \in M} Q(y|1)}.
	\end{align*}
\end{proof}

Next, we partition output set $\cY$ into $k$ subsets, where $\forall t \in [k]$, 
\[
M_t = \{y \in \cY | \arg \max_{i \in [k]}  Q(y|i) = t\}.
\]

Then combining~\eqref{eqn:three} and Lemma~\ref{lem:subset}, we have:
\begin{align*}
	& \sum_{y\in\cY} \sum_{i=1}^{k} \frac{(Q(y|i)-Q(y|1))^2}{\sum_{1\le i'\le \ab}Q(y|i')}  \le  (e^{\eps}-1)  +  \sum_{t \in[k]} \sum_{y \in M_t} (Q(y|t) -  Q(y|1)) \\
	\le &  (e^{\eps}-1)  + \sum_{t \in[k]} \Paren{  1 - e^{-\eps} \Paren{1- \sum_{y \in M_t} Q(y|1)}  - \sum_{y \in M_t}  Q(y|1) } \\
	= & (e^{\eps}-1)  + \sum_{t \in[k]} ( 1 - e^{-\eps}) (1 - \sum_{y \in M_t}  Q(y|1)) = O(k\eps). 
\end{align*}

\subsection{Proof of~\eqref{eqn:bound_chi_bs}} \label{sec:bound_chi_bs}

Define $Q(y|j,i) = \probof{Y = y | X = (j,i)}$, we have:
\begin{align} 
	\chi^2(Q | \cP_\cZ)  & = \sum_{y \in \cY} \frac{\sum_{j = 1}^m \frac{16k_j^2 \alpha^2 }{\Paren{\sum_{t = 1}^m k_t^2}^2}  \sum_{i \in [k_j/2]} \Paren{Q(y|j, 2i)-Q(y|j, 2i-1)}^2  }
	{ \frac{1}{k} \sum_{j = 1}^{m} \sum_{i \in [k_j]} Q(y|j,i) }  \nonumber \\
	& \le  \frac{k\alpha^2}{\sum_{t = 1}^{m} k_t^2} 
	\sum_{y \in \cY} \frac{\sum_{j = 1}^m \frac{16k_j^2 }{\sum_{t = 1}^m k_t^2}  \sum_{i \in [k_j/2]} \Paren{Q(y|j, 2i)-Q(y|j, 2i-1)}^2  }
	{ \sum_{j = 1}^{m} \sum_{i \in [k_j]} Q(y|j,i) }.  \label{eqn:chi_square1}
\end{align}
Within each block $\cX_j$, the elements satisfy classic $\eps$-LDP. It is proved in~\cite{pmlr-v99-acharya19a} that $\forall j \in [m], i \in [k_j]$,
\[
\Paren{Q(y|j, 2i)-Q(y|j, 2i-1)}^2 \le \frac{(e^\eps - 1)^2 }{k_j^2} \Paren{\sum_{i \in [k_j]} Q(y|j,i)}^2.
\]
Hence we have
\begin{align}
	\chi^2(Q | \cP_\cZ) & \le \frac{k\alpha^2 (e^\eps - 1)^2 }{\sum_{t = 1}^{m} k_t^2} 
	\sum_{y \in \cY} \frac{\sum_{j = 1}^m \frac{8k_j}{\sum_{t = 1}^m k_t^2}  \Paren{\sum_{i \in [k_j]} Q(y|j,i)}^2  }
	{ \sum_{j = 1}^{m} \sum_{i \in [k_j]} Q(y|j,i)}  \nonumber \\
	& \le \frac{k\alpha^2 (e^\eps - 1)^2 }{\sum_{t = 1}^{m} k_t^2} \sum_{y \in \cY} \sum_{j = 1}^m \frac{8k_j}{\sum_{t = 1}^m k_t^2}  \Paren{\sum_{i \in [k_j]} Q(y|j,i)} \nonumber \\
	& = \frac{k\alpha^2 (e^\eps - 1)^2 }{\sum_{t = 1}^{m} k_t^2} \sum_{j = 1}^m \frac{8k_j}{\sum_{t = 1}^m k_t^2} \sum_{y \in \cY}  \Paren{\sum_{i \in [k_j]} Q(y|j,i)} \nonumber \\
	& \le \frac{k\alpha^2 (e^\eps - 1)^2 }{\sum_{t = 1}^{m} k_t^2} \sum_{j = 1}^m \frac{8k_j}{\sum_{t = 1}^m k_t^2} \times k_j\nonumber \\
	& \le \frac{8k\alpha^2 (e^\eps - 1)^2 }{\sum_{t = 1}^{m} k_t^2} \nonumber \\
	& = O \Paren{\frac{k \alpha^2 \eps^2 }{\sum_{t = 1}^m k_t^2}}.
\end{align}